\documentclass[a4paper]{article}

\usepackage[utf8]{inputenc} 
\usepackage[T1]{fontenc} 
\usepackage{graphicx}
\usepackage{subfig}
\usepackage{amsmath,amssymb}
\usepackage{amsthm}
\usepackage{mathrsfs}
\usepackage{mathtools}
\usepackage[shortlabels]{enumitem}
\usepackage{algorithm}
\usepackage[noend]{algpseudocode}
\usepackage{thmtools}
\usepackage{setspace}
\usepackage[round]{natbib}
\usepackage[dvipsnames]{xcolor}
\usepackage{geometry}

\usepackage{bbm}
\usepackage{bm}
\usepackage{placeins}
\usepackage{accents}
\usepackage{tikz}
\usetikzlibrary{positioning}
\usetikzlibrary{arrows.meta}
\usetikzlibrary{calc}

\allowdisplaybreaks

\usepackage{microtype}
\usepackage{booktabs}

\usepackage[colorlinks=true, linkcolor=MidnightBlue, filecolor=MidnightBlue, citecolor=MidnightBlue, urlcolor=MidnightBlue,]{hyperref}

\usepackage[capitalize,noabbrev]{cleveref}

\let\Algorithm\algorithm
\renewcommand\algorithm[1][]{\Algorithm[#1]\setstretch{1.3}}

\DeclareMathOperator*{\argmax}{arg\,max}

\DeclareMathAlphabet{\mathbbold}{U}{bbold}{m}{n}

\newtheorem{assumption}{Assumption}
\newtheorem{theorem}{Theorem}

\newtheorem{lemma}{Lemma}
\newtheorem{setting}{Setting}

\begin{document}

  \title{Invariance-based dynamic regret minimization}
  \author{Margherita Lazzaretto\textsuperscript{1}\thanks{E-mail: \texttt{mala@math.ku.dk}},\,Jonas Peters\textsuperscript{2} and Niklas Pfister\textsuperscript{3}}
  \date{
        \textsuperscript{1}University of Copenhagen,\,\textsuperscript{2}ETH Z\"{u}rich,\,\textsuperscript{3}Lakera Z\"{u}rich
        } 
\maketitle

\begin{abstract}%
We consider stochastic non-stationary linear bandits where the linear parameter connecting contexts to the reward changes over time. Existing algorithms in this setting localize the policy by gradually discarding or down-weighting past data, effectively shrinking the time horizon over which learning can occur. However, in many settings historical data may still carry partial information about the reward model.
We propose to leverage such data while adapting to changes, by assuming the reward model decomposes into stationary and non-stationary components. Based on this assumption, we introduce ISD-linUCB, an algorithm that uses past data to learn invariances in the reward model and subsequently exploits them to improve online performance. We show both theoretically and empirically that leveraging invariance reduces the problem dimensionality, yielding significant regret improvements in fast-changing environments when sufficient historical data is available. 
\end{abstract}

\section{Introduction}
A stochastic contextual bandit models an online decision-making process in which an agent, over $T\in\mathbb{N}$ rounds, sequentially selects actions based on contextual information \citep[see, e.g.,][]{lattimore2020bandit}. The agent's objective is to learn a decision policy that selects actions in a way that maximizes a cumulative reward, by balancing exploration of new actions and exploitation of acquired knowledge.
Bandit algorithms provide the simplest framework for decision problems under uncertainty, where the actions of the agent do not affect the environment, often serving as a starting point for more complex models such as reinforcement learning.
The design of stochastic contextual bandit algorithms relies on assumptions that specify the class of possible reward distributions over which the agent learns.
We consider the case of linear time-varying reward functions, where the expected reward is assumed to depend linearly on some $p$-dimensional context-action feature.

Regret is the main performance measure of bandit algorithms. It corresponds to the difference between the cumulative reward of an optimal (oracle) sequence of actions and the action selected by the bandit algorithm. The analysis of linear bandit algorithms involves studying finite sample upper and lower bounds for the regret in terms of the time horizon $T$ and the dimensionality $p$. 
The standard setup assumes that the environment in which the agent acts is stationary, and is widely studied in the literature \citep[see, for example,][a detailed review of the bandit literature is provided in Appendix~\ref{sec:related works}.]{lattimore2020bandit}. \citet{dani2008stochastic} show a lower bound for the regret of the stationary stochastic linear bandit problem of $\Omega(p\sqrt{T})$. Different algorithms have been proposed in the literature achieving an upper bound that match this lower bound up to logarithmic factors, e.g., based on strategies using upper confidence bound (UCB) \citep{auer2002using,dani2008stochastic, abbasi2011improved}.
Recent works relax the stationarity assumption. In such settings, the evaluation of non-stationary bandit algorithms additionally takes into account a variation budget $B_T$, measuring how much the environment changes through the $T$ rounds.
When the underlying reward function changes through time, learning and updating a reliable policy can be challenging since the algorithm needs to continually explore the context-action space to detect and adapt to changes. 
\citet{cheung2019learning} deal with non-stationarity by using a sliding window regularized least squares estimator, \citet{russac2019weighted} weight, instead, past observations by a discounting factor; \citet{zhao2020simple} propose to periodically restart a standard linear bandit algorithm, where the restarting interval depends on $B_T$. All these algorithms achieve the same regret upper bound in terms of $p, T$ and $B_T$ given by $\tilde{O}(p^{\frac{7}{8}}T^{\frac{3}{4}}B_T^{\frac{1}{4}})$. The approach of \citet{zhao2020simple} implies in particular that the excess regret compared to the stationary case  arises from assuming a fixed linear parameter, and incurring an additional loss that depends on its variation between restarts. 

With the overall goal of adapting more rapidly to environment changes, we investigate whether parts of the non-stationary reward function remain invariant through all rounds, so that the data collected further in the past can be exploited rather than being discounted or fully discarded. Invariant information allows, for example, to learn policies that select worst-case optimal actions \citep[e.g.,][]{saengkyongam2023invariant}. Since purely invariant policies can be sub-optimal at specific time points, we are interested in learning an invariant policy that can be updated online. To do so, we rely on the invariant subspace decomposition (ISD) framework proposed by \citet{lazzaretto2025invariant} for regression in linear non-stationary settings. ISD splits the learning of the time-varying parameter into two lower-dimensional components, one of which is time-invariant of dimension $p^{\operatorname{inv}}<p$: this allows us to use all available data for the invariant component estimation and thus to reduce prediction error. A schematic representation of the proposed algorithm is shown in Figure~\ref{fig:isd_linucb_scheme}.
\begin{figure}
    \centering
\begin{tikzpicture}[
  box/.style={rounded corners=4pt},
  arrow/.style={->},
  font=\sffamily\footnotesize
]

\node[box, fill=cyan!30, minimum width=5cm, minimum height=3cm] (core) at (0,0) {};

\node[box, align=center, 
     fill=gray!20]
(inv) at ([xshift=-2pt,yshift=21pt]core)
{invariant \\ reward \\ prediction};

\node[box, align=center, 
     fill=gray!20]
(res) at ([xshift=-2pt,yshift=-23pt]core)
{residual \\ reward \\ prediction};

\node (plus) at ($(inv.south)!0.5!(res.north)$) {$+$};

\node[box, minimum width=0.9cm, minimum height=0.7cm, 
    fill=gray!20]
(ucb) at ([xshift=-25pt, yshift=-1]core.east) {UCB};

\node (isd) at ([xshift=-25pt, yshift=30]core.east) {\textbf{ISD-linUCB}};

\draw[arrow, rounded corners=6pt] (inv.east) -| (ucb.north);
\draw[arrow, rounded corners=6pt] (res.east) -| (ucb.south);

\node[box,draw] (context) at (-3.5,0)
{context};

\node[box, draw, align=center, 
]
(offline) at ([xshift=-2pt,yshift=55pt]core)
{offline data};

\node[box, draw, align=center, anchor=west]
(online) at ([xshift=8pt,yshift=0pt]core.west)
{online \\ data};

\draw[arrow] (offline.south) -- (inv.north);

\draw[arrow] (context.east) -- (online.west);
\draw[arrow, rounded corners=6pt, dashed] (online.north) |- (inv.west) node[midway, above, black!50] {(optional)};
\draw[arrow, rounded corners=6pt] (online.south) |- (res.west);

\node[box, draw] (reward) at (3.9,-0.03)
{reward};
\node at (2.9,0.2) {action};
\draw[arrow] (ucb.east) -- (reward.west);

\draw[arrow, rounded corners=6pt]
  (reward.south) -- ++(0,-1.5)
  -- ++(-5.85,0)
  --  ($(online.south west)!0.25!(online.south east)$)+(0, -0.001);
\end{tikzpicture}
    \caption{ISD-linUCB exploits historical data to improve reward predictions used by the UCB policy.
    }
    \label{fig:isd_linucb_scheme}
\end{figure}
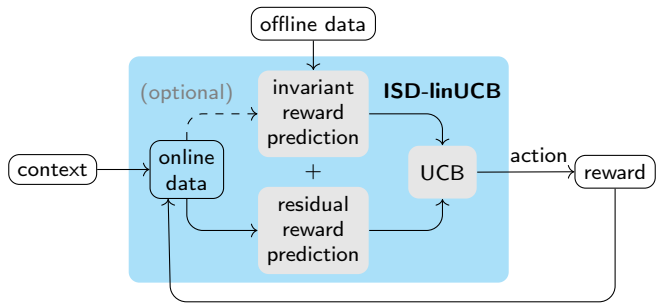

Our contributions are as follows.
\begin{enumerate}[(i), topsep=0.01cm, itemsep=-0.2cm]
    \item We propose 
    a practical novel linear contextual bandit algorithm (ISD-linUCB) that reduces online adaptation to a lower-dimensional residual subspace by exploiting the ISD framework to estimate the invariant component from historical data.
    \item We establish regret bounds scaling with the residual dimension $(p - p^{\operatorname{inv}})$ rather than $p$, yielding significant improvements in rapidly changing environments when sufficient historical data is available.
\end{enumerate}
\subsection{Notation}
For all $n\in\mathbb{N}$, we define $[n]\coloneqq\{1, \dots, n\}$ and $[-n]\coloneqq\{-n, \dots, -1\}$. For all linear subspaces $\mathcal{S}\subseteq \mathbb{R}^p$, we denote by $\Pi^{\mathcal{S}}\in\mathbb{R}^{p\times p}$ the orthogonal projection matrix onto $\mathcal{S}$ and by $\mathcal{S}^{\perp}$ the orthogonal complement of $\mathcal{S}$ in $\mathbb{R}^p$. For all $x\in\mathbb{R}^p$ and for all positive semi-definite matrices $A\in\mathbb{R}^{p\times p}$, we define the  norm $\|x\|_A\coloneqq\sqrt{x^{\top}Ax}$. We further denote by $\lambda_{\min}(A)$ and $\lambda_{\max}(A)$ the minimum and maximum eigenvalues of $A$, respectively. For all $p\in\mathbb{N}$, we denote by $I_p$ the $p$-dimensional identity matrix. We write $\tilde{O}(\cdot)$ for asymptotic bounds up to
polylogarithmic factors, i.e., $b_n=\tilde{O}(a_n)$ if there exist constants $c>0$ and $k\geq 0$ such that $b_n\leq c a_n \log(n)^k$ for sufficiently large $n$.
\section{Problem setting}
\label{sec:problem_setting}
In a linear stochastic contextual bandit, an agent sequentially observes context $X_t\in\mathcal{X}\subseteq\mathbb{R}^p$ for rounds $t\in[T]$, where each $X_t$ is drawn independently from previous contexts and has distribution $\mathbb{P}_{X_t}$. Given $X_t$, the agent selects an action $a_t\in\mathcal{A}=[K]$ and receives a noisy reward $R_t^{a_t}$. The reward for all $a\in\mathcal{A}$ is assumed to satisfy
\begin{align}
    R_t^a = \varphi(X_t, a)^{\top}\gamma_{0, t} + \epsilon_t 
    \label{eq:reward_model}
\end{align}
where $\varphi:\mathcal{X}\times\mathcal{A}\rightarrow\mathbb{R}^p$ is a known context-action feature map, $\epsilon_t$ is a conditionally zero mean $\sigma$-sub-Gaussian noise variable with respect to the $\sigma$-algebra $\mathcal{F}_t = \sigma((\varphi(X_{\tau}, a_{\tau}), R^{a_{\tau}}_{\tau})_{\tau\in[t-1]}, \varphi(X_{t}, a_{t}))$, and $\gamma_{0, t}\in\mathbb{R}^p$ is an unknown linear parameter. As in the standard linear bandit setup \citep[e.g.,][]{lattimore2020bandit}, we assume that the context-action features and the linear parameter are bounded, and define $L\coloneqq \max_{t\in\mathbb{N}}\sup_{a\in\mathcal{A}}\|\varphi(X_t, a)\|_2$ and $M\coloneqq\max_{t\in\mathbb{N}}\|\gamma_{0,t}\|_2$.

Linear bandit algorithms sequentially estimate the linear parameter $\gamma_{0,t}$ and use it to select the optimal action given the observed context $X_t$. The quality of the algorithm can then be assessed by considering the \emph{dynamic regret} defined by
\begin{equation}
     \operatorname{Reg}_T \coloneqq \mathbb{E}\left[\sum_{t=1}^T (R_t^{a^*_t}-R_t^{a_t})\right]
     \label{eq:regret_definition}
 \end{equation}
where $a^*_t \coloneqq \argmax_{a\in\mathcal{A}} \varphi(X_t, a)^{\top}\gamma_{0,t}$ is the unknown optimal action. We call the regret at time $t$, $\operatorname{reg}_t\coloneqq\mathbb{E}(R_t^{a^*_t}-R_t^{a_t})$, \emph{instantaneous regret}.

We consider a non-stationary setup in which both the parameter $\gamma_{0,t}$ and the context distribution $\mathbb{P}_{X_t}$ may vary across rounds. To be able to effectively estimate $\gamma_{0, t}$ in the non-stationary setting, we make two assumptions: (1) We assume $\gamma_{0,t}$ can be decomposed into a varying and an invariant part and (2) we assume that the varying part changes slowly so it can be seen as approximately constant within short time-intervals. To formalize (1), we introduce Assumption~\ref{ass:isd} below, which is based on the invariant subspace decomposition framework proposed by \citet{lazzaretto2025invariant} adapted to the linear bandit setup.

\begin{assumption}[Invariant subspace decomposition (ISD)]
\label{ass:isd}
There exists a non-degenerate \emph{invariant subspace decomposition} of $\mathbb{R}^p$ for the time-varying linear bandit model \eqref{eq:reward_model}, that is, a partition of $\mathbb{R}^p$ into two linear subspaces $(\mathcal{S}^{\operatorname{inv}}, \mathcal{S^{\operatorname{res}}})$ such that $\operatorname{dim}(\mathcal{S}^{\operatorname{inv}})\coloneqq p^{\operatorname{inv}}<p$ and $\operatorname{dim}(\mathcal{S}^{\operatorname{res}})\coloneqq p^{\operatorname{res}} = p-p^{\operatorname{inv}}$, satisfying
\begin{enumerate}[i), topsep=0.01cm]
	\item $\mathcal{S}^{\operatorname{inv}} = (\mathcal{S^{\operatorname{res}}})^{\perp}$;
	\item  $\forall t\in\mathbb{N}$, $\forall x\in\mathcal{X}$, $\forall a\in\mathcal{A}$: $\operatorname{Cov}(\Pi^{\mathcal{S}^{\operatorname{inv}}}\varphi(x, a), \Pi^{\mathcal{S}^{\operatorname{res}}}\varphi(x, a))=0$;
	\item $\forall t\in\mathbb{N}$, $\exists\beta^{\operatorname{inv}}\in\mathcal{S}^{\operatorname{inv}}$, $\exists\delta^{\operatorname{res}}_t\in\mathcal{S}^{\operatorname{res}}$: $\gamma_{0,t}=\beta^{\operatorname{inv}}+\delta^{\operatorname{res}}_t$; 
	\item $\forall$ partitions $(\tilde{\mathcal{S}}^{\operatorname{inv}}, \tilde{\mathcal{S}}^{\operatorname{res}})$ of $\mathbb{R}^p$ satisfying (i)---(iii): $\operatorname{dim}(\tilde{\mathcal{S}}^{\operatorname{inv}})\le\operatorname{dim}(\mathcal{S}^{\operatorname{inv}})$. 
\end{enumerate}
\end{assumption}
The spaces $\mathcal{S}^{\operatorname{inv}}$ and $\mathcal{S}^{\operatorname{res}}$ are called \emph{invariant} and \emph{residual subspace}, and $\beta^{\operatorname{inv}}$ and $\delta^{\operatorname{res}}_t$ are called \emph{invariant} and \emph{residual component} of $\gamma_{0,t}$, respectively. \citet{lazzaretto2025invariant} show that both the invariant and the residual component can be expressed as the least squares solution to the regression problem in the corresponding subspace.

Using Assumption~\ref{ass:isd}, we can now formalize the bandit setting.
We decompose the problem into an offline phase followed by a shorter online phase.
\begin{setting}
\label{set:hybrid}
    We have access to $T_0\in\mathbb{N}$ offline observations $(\varphi(X_t, a_t), a_t, R_t^{a_t})_{t\in[-T_0]}$, collected by a bandit agent interacting with a changing environment satisfying model \eqref{eq:reward_model}. The online time horizon $T\in\mathbb{N}$ is such that, for all $t\in[T]$, the parameter $\gamma_{0,t}$ in model \eqref{eq:reward_model} is constant and Assumption~\ref{ass:isd} holds.
\end{setting}
The constant $\gamma_{0,t}$ assumption simplifies the analysis and is motivated by settings of other non-stationary bandit algorithms, where $T$ may represent a single epoch in restarting algorithms \citep{zhao2020simple} or a sliding window \citep{russac2019weighted}. In practice, our proposed algorithm sequentially moves previously seen observations to the offline data.
This framing also facilitates comparison with stationary linear bandit algorithms.

While for stationary bandits the time horizon $T$ tends to dominate the bounds, for non-stationary bandits the $p$ term becomes more dominant because time horizons are short due to the shifting distributions. By leveraging the offline data, we are able to reduce the upper bound from $\tilde{O}(p\sqrt{T})$ to $\tilde{O}(p^{\operatorname{res}}\sqrt{T})$ where $p^{\operatorname{res}}$ is the dimension of the residual subspace (in Appendix~\ref{sec:lower_bound} we show that under Assumption~\ref{ass:isd} the lower bound for the regret of a linear bandit algorithm is $\Omega(p^{\operatorname{res}}\sqrt{T})$).

We begin by recalling some concepts for stationary linear bandits relevant for our work in Section~\ref{sec:prelim}. We then introduce our method (Section~\ref{sec:algorithm}) and analyze it in two steps: first assuming the decomposition $(\mathcal{S}^{\operatorname{inv}},\mathcal{S}^{\operatorname{res}})$ is known (Section~\ref{sec:oracle_subs}), then extending the analysis to the case where it is estimated  from data (Section~\ref{sec:sub_dec_error}).

\subsection{Regret analysis for stationary linear bandits}
 \label{sec:prelim}
 
 In a linear contextual bandit algorithm, the exploration is normally based on the uncertainty in the linear parameter estimation. More specifically, let $\hat{\gamma}_t$ denote an estimate of $\gamma_{0,t}$ obtained as the solution to ridge regression with regularization parameter $\lambda>0$, using observations $(X_{\tau}, R_{\tau}^{a_{\tau}})_{\tau=1}^{t-1}$. The instantaneous regret is usually upper bounded with high probability by a term proportional to the product of the estimation error on the linear parameter, $\|\hat{\gamma}_t-\gamma_{0,t}\|_{\hat{\Sigma}_{t-1}}$, and the context-action features norm, $\|\varphi(X_t, a_t)\|_{\hat{\Sigma}_{t-1}^{-1}}$, normalized by $\hat{\Sigma}_{t-1}$
 and its inverse, respectively, where for all $t\in[T]$ we define the regularized sample covariance matrix  \begin{equation}
     \hat{\Sigma}_{t} \coloneqq \lambda I_p + \sum_{\tau=1}^{t}\varphi(X_{\tau}, a_{\tau})\varphi(X_{\tau}, a_{\tau})^{\top}.
     \label{eq:cov_matrix_t}
 \end{equation}
  This is the case, for example, both for the upper confidence bound based algorithm by \citet{abbasi2011improved} (LinUCB) and for the Thompson sampling based algorithm by \citet{agrawal2013thompson} (LinTS). 
Using the explicit expression for the linear parameter estimator and the triangle inequality, it holds that
\begin{align}
     \|\hat{\gamma}_t-\gamma_{0,t}\|_{\hat{\Sigma}_{t-1}} \le &\|\sum_{\tau=1}^{t-1} \varphi(X_{\tau}, a_{\tau})^{\top}\epsilon_{\tau}\|_{\hat{\Sigma}_{t-1}^{-1}} + \sqrt{\lambda}\|\gamma_{0,t}\|_2.
     \label{eq:gamma_hat_conf_rad}
\end{align}
A result that plays a key role in bounding the linear parameter estimation error is given by the following lemma by \citet{abbasi2011improved}. 
\begin{lemma}[\citet{abbasi2011improved}, Theorem 1]
\label{lemma:self_norm_bound}
    Let $\{F_t\}_{t=0}^\infty$ be a filtration. Let $\{\epsilon_t\}_{t=1}^\infty$ be a real-valued stochastic process such that $\epsilon_t$ is $F_t$-measurable and sub-Gaussian conditionally on $F_{t-1}$ with parameter $\sigma>0$. Let $\{\varphi(X_t, a_t)\}_{t=1}^\infty$ be an $\mathbb{R}^p$-valued stochastic process such that $\varphi(X_t, a_t)$ is $F_{t-1}$-measurable. Let $\lambda\in\mathbb{R}$ be a strictly positive constant. Then, for all $\eta\in(0, 1)$, it holds with probability at least $1-\eta$ that, for all $t\in\mathbb{N}$,
    \begin{equation*}
        \|\sum_{\tau=1}^t\varphi(X_{\tau}, a_{\tau})\epsilon_{\tau}\|^2_{\hat{\Sigma}_t^{-1}} \le 2 \sigma^2 \log\left(\tfrac{1}{\eta}\sqrt{\tfrac{\operatorname{det}(\hat{\Sigma}_t)}{\operatorname{det}(\lambda I_p)}}\right).
    \end{equation*}
    Moreover, if, for all $t\in[T]$, $\|\varphi(X_t, a_t)\|_2\le L$, then
    \begin{equation*}
        \log\left(\tfrac{\operatorname{det}(\hat{\Sigma}_t)}{\operatorname{det}(\lambda I_p)}\right)\le p\log\left(1+\tfrac{tL^2}{\lambda p}\right).
    \end{equation*}
\end{lemma}
Lemma~\ref{lemma:self_norm_bound} implies that
$\|\hat{\gamma}_t-\gamma_{0,t}\|_{\hat{\Sigma}_{t-1}}$ is $\tilde{O}(\sqrt{p})$, where the dependence on $p$ is determined by the dimensionality of the matrix $\hat{\Sigma}_t$.
 The sum over the time horizon of the squared norm of the context-action features, $\sum_{t=1}^T\|\varphi(X_t, a_t)\|^2_{\hat{\Sigma}_{t-1}^{-1}}$, is also shown to be $\tilde{O}(p)$, depending again on the context-action features dimensionality.
This term appears under square root when bounding the cumulative regret, leading to an additional $\tilde{O}(\sqrt{p})$ term in the final bound. Our goal is to show that, in Setting~\ref{set:hybrid}, we can reduce the regret to depend on $p^{\operatorname{res}}$ rather than $p$, up to some term that becomes negligible for large enough $T_0$. In the following, we focus on the LinUCB analysis, but a similar reasoning could apply for other linear bandit algorithms based on least squares estimation.

\section{ISD-linUCB algorithm}
\label{sec:algorithm}
Assuming Setting~\ref{set:hybrid}, we propose an algorithm that is split into an offline phase for the estimation of $(\mathcal{S}^{\operatorname{inv}}, \mathcal{S}^{\operatorname{res}})$ and $\beta^{\operatorname{inv}}$, followed by an online bandit phase adapting to the non-stationarity.

The offline phase relies on the $T_0$ historical observations from Setting~\ref{set:hybrid}. To estimate $(\mathcal{S}^{\operatorname{inv}}, \mathcal{S}^{\operatorname{res}})$, \citet{lazzaretto2025invariant} estimate an orthonormal matrix $U$ that jointly block diagonalizes the covariance matrices $(\operatorname{Var}(\varphi(X_t, a_t)))_{t\in[-T_0]}$. $U$ can be partitioned into two submatrices $U^{\operatorname{inv}}\in\mathbb{R}^{p\times p^{\operatorname{inv}}}$ and $U^{\operatorname{res}}\in\mathbb{R}^{p\times p^{\operatorname{res}}}$ whose columns form a basis for $\mathcal{S}^{\operatorname{inv}}$ and $\mathcal{S}^{\operatorname{res}}$, respectively. Then, the orthogonal projection matrices onto $\mathcal{S}^{\operatorname{inv}}$ and $\mathcal{S}^{\operatorname{res}}$ are $\Pi^{\mathcal{S}^{\operatorname{inv}}}=U^{\operatorname{inv}}(U^{\operatorname{inv}})^{\top}$ and $\Pi^{\mathcal{S}^{\operatorname{res}}}=U^{\operatorname{res}}(U^{\operatorname{res}})^{\top}$.
Under Assumption~\ref{ass:isd}, for all $t\in[T]$ we can rewrite \eqref{eq:reward_model} as
\begin{align}
    R^{a_t}_t & = 
    \varphi(X_t, a_t)^{\top}\Pi^{\mathcal{S}^{\operatorname{inv}}}\beta^{\operatorname{inv}} + \varphi(X_t, a_t)^{\top}\Pi^{\mathcal{S}^{\operatorname{res}}}\delta^{\operatorname{res}}_t + \epsilon_t, \label{eq:sep_reward}
\end{align}
where $\operatorname{Cov}( \Pi^{\mathcal{S}^{\operatorname{inv}}}\varphi(X_t, a_t), \Pi^{\mathcal{S}^{\operatorname{res}}}\varphi(X_t, a_t))=0$. 
The subspace decomposition therefore allows to estimate $\beta^{\operatorname{inv}}$ and $\delta^{\operatorname{res}}_t$ separately. 
For any index set $\mathcal{T} \subseteq [-T_0] \cup [T]$, we define the sample covariance matrix
\begin{equation*}
    \hat{\Sigma}_{\mathcal{T}} \coloneqq  \sum_{\tau \in \mathcal{T}} \varphi(X_{\tau}, a_{\tau}) \varphi(X_{\tau}, a_{\tau})^{\top}.
    \label{eq:hist_sample_cov}
\end{equation*}
We make the following assumption to ensure that $\hat{\Sigma}_{[-T_0]}$,
is strictly positive definite.
\begin{assumption}
\label{ass:min_eig}
    The policy used to collect the $T_0$ observations is such that $\exists\lambda_0>0$ such that
    $\lambda_{\min}(\frac{1}{T_0} \hat{\Sigma}_{[-T_0]})\ge \lambda_0$ almost surely.
\end{assumption}
 This means in particular that the policy used in the collection of the offline data has explored the context-action feature space sufficiently well in all directions. This is also required to be able to estimate the invariant subspace decomposition in the first place \citep[see][]{lazzaretto2025invariant}.

We denote an estimate of $U$ by $\hat{U}=[\hat{U}^{\operatorname{inv}}, \hat{U}^{\operatorname{res}}]$. 
 Under Assumption~\ref{ass:min_eig}, we can estimate the invariant component $\beta^{\operatorname{inv}}$ as the OLS solution in the invariant subspace, using the estimated $\hat{U}^{\operatorname{inv}}$ and the offline observations, that is, 
\begin{equation}
\label{eq:beta_inv_est}
    \hat{\beta}^{\operatorname{inv}}  \coloneqq \hat{U}^{\operatorname{inv}}(\tilde{\Sigma}^{\operatorname{inv}}_{[-T_0]})^{-1} \sum_{t\in[-T_0]}(\hat{U}^{\operatorname{inv}})^{\top}\varphi(X_t, a_t)R_t^{a_t},
\end{equation}
where $\tilde{\Sigma}^{\operatorname{inv}}_{[-T_0]}\coloneqq (\hat{U}^{\operatorname{inv}})^{\top}\hat{\Sigma}_{[-T_0]}\hat{U}^{\operatorname{inv}}$. We define a confidence set around $\hat{\beta}^{\operatorname{inv}}$ by
\begin{equation}
    \hat{\mathcal{C}}^{\beta} \coloneqq \{\beta\in\hat{\mathcal{S}}^{\operatorname{inv}}\mid \|\hat{\beta}^{\operatorname{inv}}-\beta\|_{\hat{\Sigma}_{[-T_0]}}^2\le\hat{\rho}^{\operatorname{inv}}_{T_0}(\eta, L, M) \},
    \label{eq:beta_confidence_set}
\end{equation}
where $\hat{\rho}^{\operatorname{inv}}_{T_0}(\eta, L, M)$ is defined in \eqref{eq:beta_inv_radius_full} in Appendix~\ref{sec:proof_full_isd} such that the set contains $\hat{\Pi}^{\mathcal{S}^{\operatorname{inv}}}\beta^{\operatorname{inv}}$ with probability at least $1-\eta$. 
As new observations become available, the estimate $\hat{\beta}^{\operatorname{inv}}$ and the confidence set in in \eqref{eq:beta_confidence_set} may be recomputed online.

At round $t\in[T]$ of the online ISD-linUCB algorithm, we estimate the residual component as
\begin{equation*}
    \hat{\delta}^{\operatorname{res}}_t \coloneqq \hat{U}^{\operatorname{res}}(\tilde{\Sigma}^{\operatorname{res}}_{t-1})^{-1}\sum_{\tau=1}^{t-1}(\hat{U}^{\operatorname{res}})^{\top}\varphi_{\tau}(R^{a_{\tau}}_{\tau}-\varphi_{\tau}^{\top}\hat{\beta}^{\operatorname{inv}})
\end{equation*}
where $\varphi_{\tau}\coloneqq\varphi(X_{\tau},a_{\tau})$ and for all $t\in[T]$, $\tilde{\Sigma}^{\operatorname{res}}_{t}\coloneqq (\hat{U}^{\operatorname{res}})^{\top}\hat{\Sigma}_{t}\hat{U}^{\operatorname{res}}$ and $\hat{\Sigma}_t$ is defined as in \eqref{eq:cov_matrix_t}. 
We then define a confidence set around $\hat{\delta}_t^{\operatorname{res}}$ by
\begin{equation*}
    \hat{\mathcal{C}}^{\delta}_t \coloneqq \{\delta\in\hat{\mathcal{S}}^{\operatorname{res}}\mid \|\hat{\delta}^{\operatorname{res}}_t-\delta\|_{\hat{\Sigma}_{t-1}}^2\le\hat{\rho}^{\operatorname{res}}_t(\eta, L, M) \},
\end{equation*}
where $\hat{\rho}_t^{\operatorname{res}}(\eta,L,M)$ is defined in \eqref{eq:delta_res_radius_full} in Appendix~\ref{sec:proof_full_isd} such that the set contains $\hat{\Pi}^{\mathcal{S}^{\operatorname{res}}}\delta^{\operatorname{res}}_t$ with probability at least $1-\eta$, and choose an action based on $\hat{\mathcal{C}}^{\beta}\oplus\hat{\mathcal{C}}^{\delta}_t$. Assumption~\ref{ass:isd} ensures that there is no bias introduced due to omitting the context-action features in the residual subspace when estimating ${\beta}^{\operatorname{inv}}$ (and vice versa for $\delta^{\operatorname{res}}_t$) and therefore that $\hat{\mathcal{C}}^{\beta}$ and $\hat{\mathcal{C}}^{\delta}_t$ have the desired coverage.

The full ISD-linUCB algorithm for a single online step, with and without updating the invariant component, is provided in Algorithm~\ref{alg:isd_linucb}.
\begin{algorithm}[t]
    \caption{ISD-linUCB (iteration at time $t$)}
    \begin{algorithmic}
    \State \textbf{Input:} $(\varphi(X_{\tau}, a_{\tau}), a_{\tau}, R^{a_{\tau}}_{\tau})_{\tau\in[t-1]\cup[-T_0]}$, $X_t$
    \State \textbf{Parameters:} $\mathcal{A}$, $\lambda$, $\eta$, $L$, $M$, \texttt{recompute}
    \vspace{-1.25ex}
    \State \hrulefill
    \State $\mathbf{X}_{t} \gets [\varphi(X_1,a_1),\dots, \varphi(X_{t-1},a_{t-1})]^{\top}$ \State $\mathbf{R}_t\gets [R_{1}^{a_1}, \dots, R_{t-1}^{a_{t-1}}]^{\top}$
    \If{\texttt{recompute} or $t=1$}
        \State $\mathcal{T}\gets [-T_0]\cup [t-1]$
        \State $\overline{\mathbf{X}}\gets [\varphi(X_{-T_0},a_{-T_0}),\dots, \varphi(X_{t-1},a_{t-1})]^{\top}$
        \State compute $\hat{U}^{\operatorname{res}}$ and $\hat{U}^{\operatorname{inv}}$ using joint block diagonalization
        \State $\overline{\mathbf{X}}^{\operatorname{inv}} \gets \overline{\mathbf{X}}\hat{U}^{\operatorname{inv}}$
        \State $\hat{\beta}^{\operatorname{inv}} \gets \hat{U}^{\operatorname{inv}}\big((\overline{\mathbf{X}}^{\operatorname{inv}})^{\top}\overline{\mathbf{X}}^{\operatorname{inv}} \big)^{-1}(\overline{\mathbf{X}}^{\operatorname{inv}})^{\top}\overline{\mathbf{R}}$
        \State $\hat{\mathcal{C}}^{\beta} \gets \{\beta\in\hat{\mathcal{S}}^{\operatorname{inv}}\mid \|\hat{\beta}^{\operatorname{inv}}-\beta\|_{\hat{\Sigma}_{\mathcal{T}}}^2\le\hat{\rho}^{\operatorname{inv}}_{|\mathcal{T}|}(\eta, L, M) \}$
    \Else
        \State load $\hat{\beta}^{\operatorname{inv}}$, $\hat{U}^{\operatorname{res}}$ and $\hat{\mathcal{C}}^{\beta}$ from previous iteration
    \EndIf
    \State $\mathbf{R}^{\operatorname{res}}_t \gets \mathbf{R}_t -\mathbf{X}_t\hat{\beta}^{\operatorname{inv}}$  
    \State $\mathbf{X}^{\operatorname{res}}_{t} \gets \mathbf{X}_{t}\hat{U}^{\operatorname{res}}$
    \State $\hat{\delta}^{\operatorname{res}}_t \gets \hat{U}^{\operatorname{res}}\big(\lambda I_{p^{\operatorname{res}}} + (\mathbf{X}^{\operatorname{res}}_{t})^{\top}\mathbf{X}^{\operatorname{res}}_t \big)^{-1}(\mathbf{X}_t^{\operatorname{res}})^{\top}\mathbf{R}_t^{\operatorname{res}}$
    \State $\hat{\mathcal{C}}^{\delta}_t \gets \{\delta\in\hat{\mathcal{S}}^{\operatorname{res}}\mid \|\hat{\delta}^{\operatorname{res}}_t-\delta\|_{\hat{\Sigma}_{t-1}}^2\le\hat{\rho}^{\operatorname{res}}_t(\eta, L, M) \}$
    \State $a_t\gets \argmax_{a\in\mathcal{A}} \max_{\gamma\in\hat{\mathcal{C}}^{\beta}\oplus\hat{\mathcal{C}}^{\delta}_t}\varphi(X_t, a)^{\top}\gamma$
    \\
    \Return $a_t$
    \end{algorithmic}
    \label{alg:isd_linucb}
\end{algorithm}
Since Setting~\ref{set:hybrid} assumes $\gamma_{0,t}$ is fixed in $[T]$, ISD-linUCB acts in $\mathcal{S}^{\operatorname{res}}$ as a standard LinUCB algorithm. When this assumption fails, Algorithm~\ref{alg:isd_linucb} can be modified to act in $\mathcal{S}^{\operatorname{res}}$ as existing non-stationary algorithms (e.g., using a sliding window),  
to improve their performance in terms of dimensionality.

\section{Regret analysis}
To motivate our approach, in Section~\ref{sec:oracle_analysis} we first analyze the regret of a simplified algorithm having oracle knowledge of the subspace decomposition $(\mathcal{S}^{\operatorname{inv}}, \mathcal{S}^{\operatorname{res}})$. We then discuss the complete regret analysis in Section \ref{sec:complete_reg_analysis}.

\subsection{Motivation: oracle ISD-linUCB}
\label{sec:oracle_analysis}
We study Algorithm~\ref{alg:isd_linucb} assuming the matrix $U=[U^{\operatorname{inv}}, U^{\operatorname{res}}]$ is known. For simplicity, we keep the notation introduced in the previous section to denote the same quantities with known $U$. We define $(\bar{\beta}_t, \bar{\delta}_t)\coloneqq\argmax_{\beta\in\hat{\mathcal{C}}^{\beta}, \delta\in\hat{\mathcal{C}}^{\delta}}\varphi(X_t, a_t)^{\top}(\beta+\delta)$ the parameters at which the upper confidence bound for the chosen action at time $t$ is achieved. 

Using the decomposition of the reward from \eqref{eq:sep_reward} and the standard LinUCB analysis \citep[see, for example,][Section 19.3]{lattimore2020bandit}, we obtain that at time $t\in[T]$, with high probability, the instantaneous regret is upper bounded by 
\begin{align}
    \operatorname{reg}_t \le \underbrace{\varphi(X_t, a_t)^{\top}U^{\operatorname{inv}}(U^{\operatorname{inv}})^{\top}(\bar{\beta}_t-\beta^{\operatorname{inv}})}_{\operatorname{reg}^{\operatorname{inv}}_t}
    + \underbrace{\varphi(X_t, a_t)^{\top}U^{\operatorname{res}}(U^{\operatorname{res}})^{\top}(\bar{\delta}_t-\delta^{\operatorname{res}}_t)}_{\operatorname{reg}^{\operatorname{res}}_t}. \label{eq:inst_reget_isd}
\end{align}
Intuitively, for larger $T_0$ the confidence set $\hat{\mathcal{C}}^{\beta}$ shrinks and $\bar{\beta}_t-\beta^{\operatorname{inv}}$ gets close to zero. Therefore, for $T_0\gg T$, $\operatorname{reg}^{\operatorname{inv}}_t$  becomes negligible with respect to $\operatorname{reg}^{\operatorname{res}}_t$. Removing the uncertainty on the invariant component implies that, in the regret analysis, the dependence on the dimension $p$ of the bandit parameter is reduced to the dimension $p^{\operatorname{res}}$ of the residual parameter.
We first study the regret assuming to also have oracle knowledge of the invariant component $\beta^{\operatorname{inv}}$ (Section~\ref{sec:oracle_subs_beta}). Then, we include in the regret analysis the estimation of the invariant component (Section~\ref{sec:oracle_subs}). 

\subsubsection{Regret analysis for oracle $(\mathcal{S}^{\operatorname{inv}}, \mathcal{S}^{\operatorname{res}}),\,\beta^{\operatorname{inv}}$}
\label{sec:oracle_subs_beta}
If we know $\beta^{\operatorname{inv}}$, then $\hat{\mathcal{C}}^{\beta}=\{\beta^{\operatorname{inv}}\}$ and $\operatorname{reg}^{\operatorname{inv}}_t=0$. Therefore, Algorithm~\ref{alg:isd_linucb} only performs the exploration in the residual subspace. As shown in Theorem~\ref{thm:oracle_regret}, its regret only depends on the uncertainty on $\hat{\delta}^{\operatorname{res}}_t$, which lies on a $p^{\operatorname{res}}$-dimensional space.

\begin{theorem} 
    Consider Setting~\ref{set:hybrid}, assume $(\mathcal{S}^{\operatorname{inv}}, \mathcal{S}^{\operatorname{res}})$ and $\beta^{\operatorname{inv}}$ are known and consider the oracle version of Algorithm~\ref{alg:isd_linucb} that uses $\beta^{\operatorname{inv}}$, $U^{\operatorname{inv}}$ and $U^{\operatorname{res}}$ instead of $\hat{\beta}^{\operatorname{inv}}$, $\hat{U}^{\operatorname{inv}}$ and $\hat{U}^{\operatorname{res}}$. Then, the regret of this oracle algorithm over a time horizon $T$ with $\eta=1/T$
    is $\tilde{O}(p^{\operatorname{res}}\sqrt{T})$.
    \label{thm:oracle_regret}
\end{theorem}
The proof of Theorem~\ref{thm:oracle_regret} uses that, under oracle knowledge, $\operatorname{reg}_t=\operatorname{reg}^{\operatorname{res}}_t$, where the dimension of $(U^{\operatorname{res}})^{\top}\varphi(X_t, a_t)$ and of $(U^{\operatorname{res}})^{\top}(\bar{\delta}_t-\delta^{\operatorname{res}}_t)$ is $p^{\operatorname{res}}$.
It then follows similar steps as the regret analysis for the standard LinUCB algorithm introduced in Section~\ref{sec:prelim}. A detailed proof is provided in Appendix~\ref{app:oracle_bound}.

\subsubsection{Estimating the invariant component from offline data}
\label{sec:oracle_subs}

When also considering the uncertainty on the invariant component estimated using $T_0$ observations, we can upper bound the regret of Algorithm~\ref{alg:isd_linucb} as follows. 

\begin{theorem}
    Consider Setting~\ref{set:hybrid}, assume Assumption~\ref{ass:min_eig} holds and that $(\mathcal{S}^{\operatorname{inv}}, \mathcal{S}^{\operatorname{res}})$ is known, and consider the oracle version of Algorithm~\ref{alg:isd_linucb} that uses $U^{\operatorname{inv}}$ and $U^{\operatorname{res}}$ instead of $\hat{U}^{\operatorname{inv}}$ and $\hat{U}^{\operatorname{res}}$. Then, the regret of this oracle algorithm over a time horizon $T$ with $\eta=1/T$ is 
    $\tilde{O}\left(\sqrt{T}\left(p^{\operatorname{res}}+ \Big(\sqrt{\frac{p^{\operatorname{inv}}}{\lambda_0}} + \frac{1}{\lambda_0}\Big)\sqrt{\frac{T}{T_0}}\right)\right)$.
    \label{thm:estimated_beta_inv}
\end{theorem}
Theorem~\ref{thm:estimated_beta_inv} implies that, whenever $T_0$ is sufficiently larger than $T$, we have an advantage in using offline data to estimate invariant information in the regret function over only relying on online data. Indeed, the larger $T_0$ is, the closer we get to the oracle bound shown in Section~\ref{sec:oracle_subs_beta} of $\tilde{O}(p^{\operatorname{res}}\sqrt{T})$. 

\begin{proof}[Proof sketch of Theorem~\ref{thm:estimated_beta_inv}]
 To bound $\operatorname{reg}^{\operatorname{inv}}_t$ we
 choose $\hat{
 \rho}^{\operatorname{inv}}_{T_0}(\eta, L, M)$ such that
 $\hat{\mathcal{C}}^{\beta}$ contains $\beta^{\operatorname{inv}}$ with probability at least $1-\eta$. 
By Assumption~\ref{ass:min_eig}, we have that $\hat{\Sigma}_{[-T_0]}\succeq \lambda_0T_0 I_p$ (with $\succeq$ denoting the Loewner order between two matrices) which implies that $\tilde{\Sigma}^{\operatorname{inv}}_{[-T_0]}\succeq \lambda_0T_0 I_{p^{\operatorname{inv}}}$. It follows that, for all $\beta\in\hat{\mathcal{C}}^{\beta}$,
\begin{align}
    \|\hat{\beta}^{\operatorname{inv}}-\beta\|_{\hat{\Sigma}_{[-T_0]}}^2= & \|(\tilde{\Sigma}^{\operatorname{inv}}_{[-T_0]})^{\frac{1}{2}}(U^{\operatorname{inv}})^{\top}(\hat{\beta}^{\operatorname{inv}}-\beta)\|_2^2 \notag\\
    \ge &
    \|\sqrt{\lambda_0T_0}(U^{\operatorname{inv}})^{\top}(\hat{\beta}^{\operatorname{inv}}-\beta)\|_2^2\notag \\
     = &
     \lambda_0T_0 \|(U^{\operatorname{inv}})^{\top}(\hat{\beta}^{\operatorname{inv}}-\beta)\|_2^2. \label{eq:sigma_2_norm_ineq}
\end{align}
Finally, this implies that
\begin{align*}
    \|(U^{\operatorname{inv}})^{\top}(\hat{\beta}^{\operatorname{inv}}-\beta)\|_2^2 \le \frac{1}{\lambda_0T_0}\hat{\rho}^{\operatorname{inv}}_{T_0}(\eta, L, M).
\end{align*}
Therefore, with probability at least $1-\eta$, 
\begin{align*}
    \operatorname{reg}_t^{\operatorname{inv}}\le 2 \|(U^{\operatorname{inv}})^{\top}\varphi(X_t, a_t)\|_2 \sqrt{\frac{\hat{\rho}^{\operatorname{inv}}_{T_0}(\eta, L, M)}{\lambda_0T_0}}.
\end{align*}
By Lemma~\ref{lemma:self_norm_bound} and Assumption~\ref{ass:min_eig}, we can define $\sqrt{\hat{\rho}^{\operatorname{inv}}_{T_0}(\eta, L, M)}$ to be $\tilde{O}(\sqrt{p^{\operatorname{inv}}}+ \sqrt{\frac{1}{\lambda_0}})$.  
The analysis of $\operatorname{reg}^{\operatorname{res}}_t$ is similar to Theorem~\ref{thm:oracle_regret}, with the addition of a term due to the introduction of $\hat{\beta}^{\operatorname{inv}}$ which is $\tilde{O}(\sqrt{{T\hat{\rho}^{\operatorname{inv}}_{T_0}(\eta, L, M)}/{(\lambda_0T_0})})$.
This implies, summing over $t\in[T]$, that the cumulative regret is $\tilde{O}\left(\sqrt{T}\left(p^{\operatorname{res}}+ \sqrt{\frac{p^{\operatorname{inv}}T}{\lambda_0 T_0}} + \frac{1}{\lambda_0}\sqrt{\frac{T}{T_0}}\right)\right)$.  For the full proof, see Appendix~\ref{sec:proof_oracle_subs}.
\end{proof}

\subsection{Accounting for errors in the subspace decomposition}
\label{sec:sub_dec_error}
When estimating $(\mathcal{S}^{\operatorname{inv}}, \mathcal{S}^{\operatorname{res}})$ using the available $T_0$ observations, we need to consider two sources of errors. \citet{lazzaretto2025invariant} propose to estimate the subspaces via joint block diagonalization, in our case of the sample covariance matrices of the context-action features through time. The procedure first finds an irreducible joint decomposition of the features into orthogonal lower dimensional subspaces such that the projections of the features onto the subspaces are pairwise uncorrelated. Then, it groups together subspaces in which the linear relationship between the reward and the features is invariant, defining $\mathcal{S}^{\operatorname{inv}}$, and the remaining ones in which the linear relationship is time-varying, defining $\mathcal{S}^{\operatorname{res}}$. 

The first source of error can occur when some directions in $\mathbb{R}^p$ are wrongly identified as invariant or time-varying. Labeling directions as time-varying when they are invariant is less problematic, as it leads to a suboptimal but not incorrect algorithm: in this case, $p^{\operatorname{inv}}$ is underestimated, preventing the decomposition from achieving its maximum benefit.
Overestimating $p^{\operatorname{inv}}$, i.e., assuming invariance of time-varying directions, implies instead that in such directions the algorithm uses the pooled estimated parameter from the $T_0$ historical observations to predict the reward, while the true parameter may have changed. Then, intuitively, ISD-linUCB suffers an additional loss scaling with how different the true parameter is in $[T]$ compared to its average in $[-T_0]$ in the wrongly labeled subspace.

The second source of error is the estimation error due to finite sample approximation. 
To estimate the matrix $U$ we first estimate $(\operatorname{Var}(\varphi(X_t, a_t)))_{t\in[-T_0]}$; in practice, it is sufficient to split $[-T_0]$ into $m$ windows and compute $\hat{\Sigma}_{\mathcal{M}_i}$, defined as in \eqref{eq:hist_sample_cov} with $\mathcal{M}_i\coloneqq\{-\frac{i}{m}T_0, \dots, -\frac{i-1}{m}T_0-1\}$, $i\in\{1, \dots, m\}$. Then, $\hat{U}$ is an estimated joint block diagonalizer for $(\hat{\Sigma}_{\mathcal{M}_i})_{i\in\{1,\dots,m\}}$.
The matrix $\hat{U}$ then enters in the estimation of both $\beta^{\operatorname{inv}}$ and  ${\delta}^{\operatorname{res}}_t$. 
Let $\hat{\gamma}^{\operatorname{ISD}} = \hat{\beta}^{\operatorname{inv}}+\hat{\delta}^{\operatorname{res}}_t$ be an estimator for $\gamma_{0,t}$ obtained using the ISD framework.
For all $t\in[T]$, we can express the estimation error as 
\begin{align}
    \hat{\gamma}^{\operatorname{ISD}}_t - \gamma_{0,t} & = \hat{\beta}^{\operatorname{inv}} - \beta^{\operatorname{inv}} + \hat{\delta}^{\operatorname{res}}_t - \delta^{\operatorname{res}}_t \notag \\
    & = (\hat{\beta}^{\operatorname{inv}} - \hat{\Pi}^{\mathcal{S}^{\operatorname{inv}}}\beta^{\operatorname{inv}})
    + (\hat{\Pi}^{\mathcal{S}^{\operatorname{inv}}} - \Pi^{\mathcal{S}^{\operatorname{inv}}})\beta^{\operatorname{inv}} \notag\\
    &\quad + (\hat{\delta}^{\operatorname{res}}_t - \hat{\Pi}^{\mathcal{S}^{\operatorname{res}}}\delta^{\operatorname{res}}_t) 
    + (\hat{\Pi}^{\mathcal{S}^{\operatorname{res}}} - \Pi^{\mathcal{S}^{\operatorname{res}}})\delta^{\operatorname{res}}_t.
    \label{eq:gamma_hat_isd_split}
\end{align}
The first and third term correspond to the estimation error for the invariant and residual component on the estimated subspaces. The second and fourth term only depend on the subspaces estimation error. 
When bounding the regret of Algorithm~\ref{alg:isd_linucb}, we assume that the subspace decomposition is correctly estimated and only take into account the finite sample error in estimating $(\mathcal{S}^{\operatorname{inv}}, \mathcal{S}^{\operatorname{res}})$, starting from the decomposition in~\eqref{eq:gamma_hat_isd_split}.

\subsection{Complete regret analysis}
\label{sec:complete_reg_analysis}
To analyze the terms in \eqref{eq:gamma_hat_isd_split} we need to quantify the subspace estimation error. A common way to do this is to consider
the distance between the true and the estimated subspace, which can be described through the notion of principal angles.  
Let $U^{\mathcal{S}}\in\mathbb{R}^{p\times k}$ be the submatrix of $U$ whose columns span the subspace $\mathcal{S}$ of $\mathbb{R}^p$, and let $\hat{U}^{\mathcal{S}}$ be its estimate obtained through $\hat{U}$, whose columns span $\hat{\mathcal{S}}$. The principal angles between $\mathcal{S}$ and $\hat{\mathcal{S}}$ are characterized as the inverse cosine of the nonzero singular values of the matrix $(\hat{U}^{\mathcal{S}})^{\top}U^{\mathcal{S}}$ (intuitively, the cosine of the angle between two vectors is given by their normalized inner product). Let $\Theta^{\mathcal{S}}$ denote the diagonal matrix with the principal angles between $\mathcal{S}$ and $\hat{\mathcal{S}}$ on the diagonal. 
Then, we measure the distance between $\mathcal{S}$ and $\hat{\mathcal{S}}$ by $\|\sin \Theta^{\mathcal{S}}\|_{\operatorname{op}}$, i.e., the largest principal angle between $\mathcal{S}$ and $\hat{\mathcal{S}}$ \citep[see for example Theorem 4.5 by][]{stewart1990matrixperturbation}, which also equals $\|\hat{\Pi}^{\mathcal{S}}-\Pi^{\mathcal{S}}\|_{\operatorname{op}}$ \citep[Corollary 4.6]{stewart1990matrixperturbation}. We denote by $\Delta \Pi$ such distance for the invariant and residual subspaces (this must be the same for both subspaces, see \citet[Preliminaries---Corollary 5.4]{stewart1990matrixperturbation}), and introduce the following assumption for the subspace decomposition error.
\begin{assumption}
\sloppy
    For all $\eta\in(0,1)$, it holds with probability at least $1-\eta$ that $\Delta\Pi$ is 
    $O(\sqrt{\log (p/\eta) /T_0})$.
    \label{ass:subspace_dec_error}
\end{assumption}
Assumption~\ref{ass:subspace_dec_error} can be justified using the Davis-Kahan theorem \citep[see, for example,][]{yu2015useful} and by concentration results for sample covariance matrices \citep[see, for example,][Sections 5.4 and 5.6]{vershynin2018high}. 
In more detail, the Davis-Kahan theorem provides an upper bound for $\|\sin \Theta^{\mathcal{S}}\|_{\operatorname{op}}$ in terms of estimation error on a matrix $M\in\mathbb{R}^{p\times p}$ such that the columns of $U$ are eigenvectors for $M$. In our case, $U$ jointly block diagonalizes $(\operatorname{Var}(\varphi(X_t, a_t)))_{t\in[-T_0]}$, and for all $t\in[-T_0]$ the span of the columns of $U^{\operatorname{inv}}$ and $U^{\operatorname{res}}$ coincide with the union of subsets of eigenspaces of $\operatorname{Var}(\varphi(X_t, a_t))$. We estimate $\hat{U}$ through the matrices $(\hat{\Sigma}_{\mathcal{M}_i})_{i=1,\dots,m}$, each computed using $T_0/m$ observations. As $\operatorname{Var}(\varphi(X_t, a_t))$ is not constant with $t$, we cannot use exact concentration results; with Assumption~\ref{ass:subspace_dec_error}, we assume that $\Delta\Pi$ scales as the finite-sample error for estimating a fixed $p$-dimensional covariance matrix from $T_0$ observations (we verify it empirically in Appendix~\ref{sec:projection_err}). 

The sum of the second and fourth term in \eqref{eq:gamma_hat_isd_split} equals $\Delta\Pi\gamma_{0,t}$. Under Assumption~\ref{ass:subspace_dec_error}, this quantity is, with probability at least $1-\eta$, $O(\sqrt{\log(p/\eta)/T_0})$. Moreover,
under Assumption~\ref{ass:subspace_dec_error}, we can show the following high probability bound for the estimation error on $\hat{\beta}^{\operatorname{inv}}$ in the estimated invariant subspace. 
\begin{lemma} Consider Setting~\ref{set:hybrid} and assume Assumptions~\ref{ass:min_eig} and \ref{ass:subspace_dec_error} hold. Then, for all $\eta\in(0, 1)$ it holds with probability at least $1-\eta$ that $\|\hat{\beta}^{\operatorname{inv}}-\hat{\Pi}^{\mathcal{S}^{\operatorname{inv}}}\beta^{\operatorname{inv}}\|_{\hat{\Sigma}_{[-T_0]}}$ is
    \begin{align*}
        & O\left(\sqrt{p^{\operatorname{inv}}\log(\tfrac{1}{p^{\operatorname{inv}}\lambda_0})+\log(\tfrac{1}{\eta})}\right)
        +  O\left(\sqrt{p^{\operatorname{inv}}\log(\tfrac{p}{\eta})}\right) 
        +  O\left(\sqrt{\tfrac{1}{\lambda_0}\log(\tfrac{p}{\eta})}\right).
    \end{align*}
    \label{lem:beta_inv_bound}
\end{lemma}
\begin{proof}[Proof sketch]
We can decompose the error into three terms as follows
\begin{align*}
     \|\hat{\beta}^{\operatorname{inv}} - \hat{\Pi}^{\mathcal{S}^{\operatorname{inv}}}\beta^{\operatorname{inv}}\|_{\hat{\Sigma}_{[-T_0]}}   
     \le & \|(\hat{U}^{\operatorname{inv}})^{\top}\hat{\Sigma}_{[-T_0]}(\Pi^{\mathcal{S}^{\operatorname{inv}}}-\hat{\Pi}^{\mathcal{S}^{\operatorname{inv}}})\beta^{\operatorname{inv}}\|_{(\tilde{\Sigma}^{\operatorname{inv}}_{[-T_0]})^{-1}} \\
     + &\|(\hat{U}^{\operatorname{inv}})^{\top}\sum_{t=-T_0}^{-1} \varphi(X_t, a_t)\varphi(X_t, a_t)^{\top}\delta^{\operatorname{res}}_t\|_{(\tilde{\Sigma}^{\operatorname{inv}}_{[-T_0]})^{-1}} \\
     + &\|(\hat{U}^{\operatorname{inv}})^{\top}\sum_{t=-T_0}^{-1} \varphi(X_t, a_t)\epsilon_t\|_{(\tilde{\Sigma}^{\operatorname{inv}}_{[-T_0]})^{-1}}.
\end{align*}
The last term can be analyzed as in the oracle case (Theorem~\ref{thm:estimated_beta_inv}).
The first two terms are introduced due to the misalignment between the true and the estimated subspaces. Under Assumption~\ref{ass:subspace_dec_error}, for all $\eta\in(0,1)$, with probability at least $1-\eta$, the first is $O(\sqrt{p^{\operatorname{inv}}\log(\frac{p}{\eta})})$ and the second is $O(\sqrt{\frac{1}{\lambda_0}\log(\frac{p}{\eta})})$. For the full proof, see Appendix~\ref{sec:proof_full_isd}.
\end{proof}
We can further obtain an upper bound for $\|\hat{\beta}^{\operatorname{inv}} - \hat{\Pi}^{\mathcal{S}^{\operatorname{inv}}}\beta^{\operatorname{inv}}\|_2$ by multiplying all the factors in Lemma~\ref{lem:beta_inv_bound} by $\frac{1}{\sqrt{\lambda_0T_0}}$ (see \eqref{eq:sigma_2_norm_ineq}).

For the residual component, we can similarly bound $\|\hat{\delta}^{\operatorname{res}}_t-\hat{\Pi}^{\mathcal{S}^{\operatorname{res}}}\delta^{\operatorname{res}}_t\|_{\hat{\Sigma}_{t-1}}$ using Assumption~\ref{ass:subspace_dec_error} as follows.
\begin{lemma}
    Consider Setting~\ref{set:hybrid} and assume Assumptions~\ref{ass:min_eig} and \ref{ass:subspace_dec_error} hold. Then, for all $t\in[T]$ and for all $\eta\in(0,1)$ it holds with probability at least $1-\eta$ that $\|\hat{\delta}^{\operatorname{res}}_t-\hat{\Pi}^{\mathcal{S}^{\operatorname{res}}}\delta^{\operatorname{res}}_t\|_{\hat{\Sigma}_{t-1}}$ is 
    \begin{align*}
        O\left(\sqrt{p^{\operatorname{res}}\log(\tfrac{t}{p^{\operatorname{res}}})+\log(\tfrac{1}{\eta})}\right)
        + O\left(\sqrt{p^{\operatorname{res}}t}\left(\sqrt{\tfrac{\log(p/\eta)}{T_0}} + \|\hat{\beta}^{\operatorname{inv}}-\beta^{\operatorname{inv}}\|_2\right)\right).
    \end{align*}
    \label{lem:delta_res_bound}
\end{lemma} 
The first term is the same that appears in the case of oracle subspaces. The second term is instead due to the error in the estimation of the subspaces. 
\begin{proof}[Proof sketch]
We can upper bound $\|\hat{\delta}^{\operatorname{res}}_t-\hat{\Pi}^{\mathcal{S}^{\operatorname{res}}}\delta^{\operatorname{res}}_t\|_{\hat{\Sigma}_{t-1}}$ by the sum of the following components
    \begin{align*}
        \|\hat{\delta}^{\operatorname{res}}_t-\hat{\Pi}^{\mathcal{S}^{\operatorname{res}}}\delta^{\operatorname{res}}_t\|_{\hat{\Sigma}_{t-1}}  
         \le & \|\sum_{\tau= 1}^{t-1}( \hat{U}^{\operatorname{res}})^{\top}\varphi_\tau\varphi_\tau^{\top}(\Pi^{\mathcal{S}^{\operatorname{res}}}-\hat{\Pi}^{\mathcal{S}^{\operatorname{res}}})\delta^{\operatorname{res}}_t\|_{(\tilde{\Sigma}^{\operatorname{res}}_{t-1})^{-1}} \\
         +& \|\sum_{\tau=1}^{t-1}( \hat{U}^{\operatorname{res}})^{\top}\varphi_\tau\varphi_\tau^{\top}(\beta^{\operatorname{inv}}-\hat{\beta}^{\operatorname{inv}})\|_{(\tilde{\Sigma}^{\operatorname{res}}_{t-1})^{-1}} \\ 
         + & \|\sum_{\tau= 1}^{t-1}( \hat{U}^{\operatorname{res}})^{\top}\varphi_\tau\epsilon_{\tau}\|_{(\tilde{\Sigma}^{\operatorname{res}}_{t-1})^{-1}} + \sqrt{\lambda}\|\delta^{\operatorname{res}}_t\|_2 
    \end{align*}
    where $\varphi_\tau\coloneqq\varphi(X_{\tau}, a_{\tau})$. The analysis of the last two terms is the same as in Theorem~\ref{thm:oracle_regret} and results in such terms being, for all $\eta\in(0,1)$, with probability at least $1-\eta$, $O\left(\sqrt{\log(\frac{1}{\eta})+p^{\operatorname{res}}\log(\frac{t}{p^{\operatorname{res}}}})\right)$.
    The first two terms are introduced due to the subspace estimation error. Under Assumption~\ref{ass:subspace_dec_error}, the first term is, for all $\eta\in(0,1)$, with probability at least $1-\eta$, $O\left(\sqrt{\frac{p^{\operatorname{res}}t}{T_0}\log(\frac{p}{\eta})}\right)$ and the second is $O\left(\sqrt{p^{\operatorname{res}}t}\|\hat{\beta}^{\operatorname{inv}}- \beta^{\operatorname{inv}}\|_2\right)$. For the full proof, see Appendix~\ref{sec:proof_full_isd}.
\end{proof}

Lemma~\ref{lem:beta_inv_bound} and Lemma~\ref{lem:delta_res_bound} imply the following regret bound for the ISD-linUCB algorithm.
\begin{theorem}
\label{thm:full_regret_bound}
    Consider Setting~\ref{set:hybrid} and assume Assumptions~\ref{ass:min_eig} and \ref{ass:subspace_dec_error} hold. If $\hat{\mathcal{C}}^{\beta}$ is defined to contain $\hat{\Pi}^{\mathcal{S}^{\operatorname{inv}}}\beta^{\operatorname{inv}}$ with probability at least $1-4/T_0$ and $\hat{\mathcal{C}}^{\delta}_t$ is defined to contain $\hat{\Pi}^{\mathcal{S}^{\operatorname{res}}}\delta^{\operatorname{res}}_t$ with probability at least $1-(\frac{2}{T}+\frac{4}{T_0})$, then, with probability at least $1-(\frac{2}{T}+\frac{9}{T_0})$ the regret of ISD-linUCB is
    \begin{align*}
        \tilde{O}\left(\sqrt{T}\left(p^{\operatorname{res}} + p^{\operatorname{res}}\sqrt{\tfrac{T}{\lambda_0 T_0}}\left(\sqrt{p^{\operatorname{inv}}}+\sqrt{\max\{\lambda_0, \tfrac{1}{\lambda_0}\}}\right)\right)\right).
    \end{align*}
\end{theorem}
Theorem~\ref{thm:full_regret_bound} shows in particular that, for $T_0$ sufficiently larger than $T$ (e.g., $T_0=\Omega(T^{1+\varepsilon})$ for some $\epsilon>0$), the regret bound for ISD-linUCB is dominated by $p^{\operatorname{res}}\sqrt{T}$.

\section{Simulation experiments}
To support our theoretical results, we present several simulation experiments. 
We first apply the ISD-linUCB algorithm  with oracle subspace knowledge 
in Section~\ref{sec:simulations_oracle_subs}, and show that
the regret indeed scales as $p^{\operatorname{res}}\sqrt{T}$, 
supporting the result that if $T_0$ grows faster than $T$, the regret indeed scales as $p^{\operatorname{res}}\sqrt{T}$.
We then include the estimation of the subspace decomposition in the algorithm (Section~\ref{sec:simulations_est_isd}), and show that for increasing $T_0$ its performance gets closer to the one of ISD-linUCB with oracle knowledge on subspaces.  In all simulations, we provide a comparison with the standard LinUCB algorithm since, in Setting~\ref{set:hybrid} we assume $\gamma_{0,t}$ to be fixed in [T] (in this setting, the performance of other non-stationary algorithms is comparable to the one of LinUCB, see Appendix~\ref{sec:non-stationary_comp}).

\subsection{Oracle subspace decomposition}
\label{sec:simulations_oracle_subs}
We consider a setting with $T_0=2000$, $T=100$ and $|\mathcal{A}|=5$ and with known $U^{\operatorname{inv}}, U^{\operatorname{res}}$. The matrix $U$ is sampled as a random orthonormal matrix, and we generate the covariance matrix of context-action features as $U\tilde{V}_t U^{\top}$, where $\tilde{V}_t$ is a block diagonal matrix with two blocks of dimensions $p^{\operatorname{inv}}, p^{\operatorname{res}}$. The entries of $(U^{\operatorname{inv}})^{\top}\beta^{\operatorname{inv}}$ are sampled uniformly in $(0.5, 1.5)$. The same is done for the initial values of $(U^{\operatorname{res}})^{\top}\delta^{\operatorname{res}}_t$, to which for $t\in[-T_0]$ we add $-1.5(t/T0)\sin^2(0.25it/T0+i)$, with $i\in\{1,\dots,p^{\operatorname{res}}\}$ being the entry index. The data in $[-T_0]$, used to estimate $\beta^{\operatorname{inv}}$, is the output of a bandit algorithm using a policy that chooses the actions uniformly at random. 
For $t\in[T]$, the entries of $(U^{\operatorname{res}})^{\top}\delta^{\operatorname{res}}_t$ are sampled uniformly in $(0.5, 1.5)$. We set $\lambda=0.1$ and we use $\eta=1/T$. We then run two experiments (in both, $\hat{\beta}^{\operatorname{inv}}$ is not updated online).

First, for a fixed dimension $p=10$ of the context-action features we consider values of $p^{\operatorname{res}}\in\{2, 4, 6 , 8\}$. Figure~\ref{fig:reg_residual_dimension} shows that the regret of ISD-linUCB (with oracle subspace information) grows sublinearly in $T$ (left) and (approximately) linearly in $p^{\operatorname{res}}$ (right), empirically supporting the result obtained in Theorems~\ref{thm:estimated_beta_inv}. 
\begin{figure}[btp]
    \centering
    \includegraphics[width=.7\linewidth]{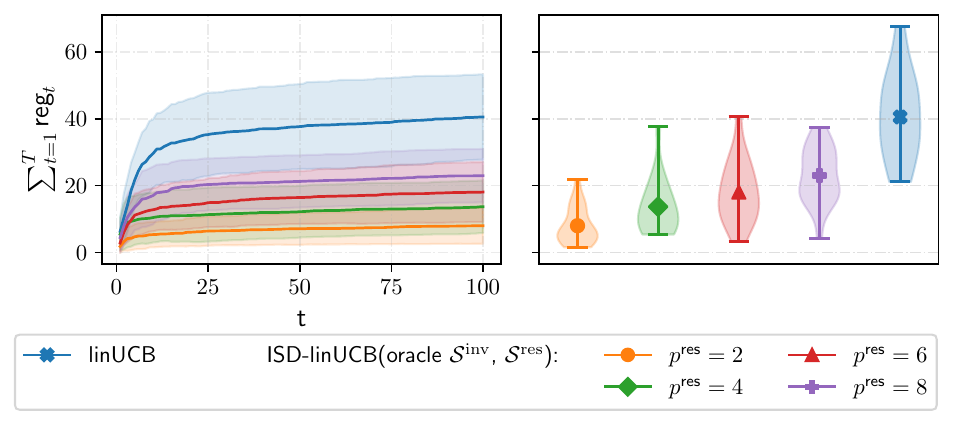}
    \caption{Regret of ISD-linUCB with oracle $(\mathcal{S}^{\operatorname{inv}}, \mathcal{S}^{\operatorname{res}})$ over $T=100$ rounds for $p^{\operatorname{res}}\in\{2,4,6,8\}$. For each $p^{\operatorname{res}}$ the experiment is repeated $20$ times. The left plot shows the average performance and the standard deviation over the $20$ repetitions, the right plot shows the distribution of the regret over the $20$ repetitions.}
    \label{fig:reg_residual_dimension}
\end{figure}
To further support these results, we run a second experiment where we consider context-action features of dimension $p$ varying between $3$ and $10$, while keeping the dimension of $\mathcal{S}^{\operatorname{res}}$ fixed to $p^{\operatorname{res}}=2$. The results are shown in Figure~\ref{fig:regret_rates_dimension_example}, which compares the cumulative regret to the one of the standard LinUCB algorithm. While the latter increases linearly with $p$, the regret of ISD-linUCB remains approximately constant as $p^{\operatorname{res}}$ is fixed. 
\begin{figure}[t]
    \centering
     \includegraphics[width=.5\linewidth]{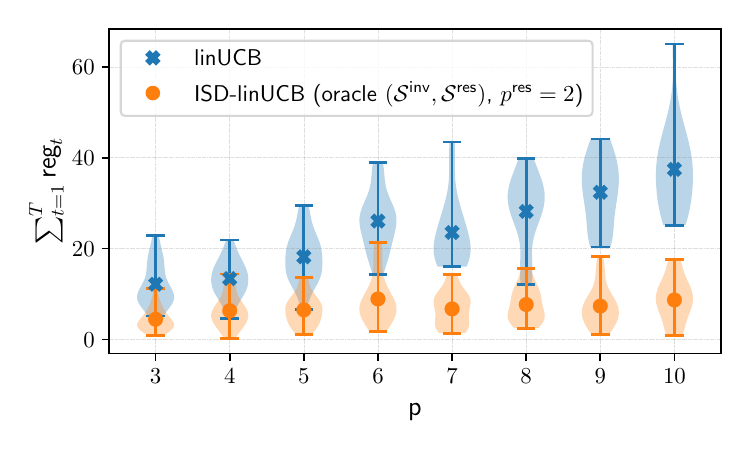}
\caption{Cumulative regret of standard LinUCB and ISD-linUCB with oracle $(\mathcal{S}^{\operatorname{inv}}, \mathcal{S}^{\operatorname{res}})$ for $T=100$ and increasing values of context-action feature dimension $p$. For ISD-linUCB, the invariant component $\beta^{\operatorname{inv}}$ is estimated using $T_0=2000$ observations.  $p^{\operatorname{inv}}$ varies from $3$ to $10$, while the $p^{\operatorname{res}}$ is fixed to $2$. For each $p$ the experiment is repeated $20$ times.
}
    \label{fig:regret_rates_dimension_example}
\end{figure}
\subsection{Estimated subspace decomposition}
\label{sec:simulations_est_isd}
We consider the same data-generating process as in Section~\ref{sec:simulations_oracle_subs}, but now include the subspace estimation using the $T_0$ observations. We set $T=500$ and consider $T_0\in\{1000, 3500, 8000\}$. Moreover, we fix the dimensions of the subspaces to $p^{\operatorname{inv}}=7$ and $p^{\operatorname{res}}=3$. 
Figure~\ref{fig:T_0_conv_example} shows the cumulative regret for increasing $T_0$ for the standard LinUCB algorithm (unaffected by $T_0$) and for the algorithms studied in Sections~\ref{sec:oracle_subs_beta}, \ref{sec:oracle_subs} and \ref{sec:sub_dec_error}, supporting the presented theoretical results. Indeed, for increasing $T_0$, the regret of the ISD-linUCB algorithm estimating $(\mathcal{S}^{\operatorname{inv}}, \mathcal{S}^{\operatorname{res}})$ and $\beta^{\operatorname{inv}}$ using the $T_0$ observations gets closer to the one of an algorithm with oracle knowledge. Moreover, the regret of the algorithm using the estimated ISD is lower than the standard LinUCB one for all considered values of $T_0$. 
\begin{figure}[t]
    \centering
    \includegraphics[width=.6\linewidth]{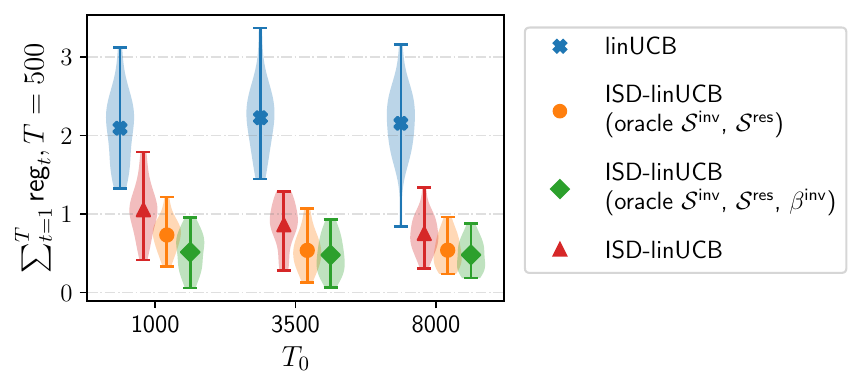}
    \caption{Cumulative regret for $T_0\in\{1000, 3500, 8000\}$ (20 repetitions) for ISD-linUCB, in comparison with the same algorithm having oracle information and with the standard linUCB alogorithm (unaffected by $T_0$). For increasing $T_0$, the regret of ISD-linUCB gets closer to the one of the oracle version.}
    \label{fig:T_0_conv_example}
\end{figure}
\section{Conclusions}
Exploiting invariances in non-stationary environments allows linear contextual bandit algorithms to adapt more efficiently to changes in the reward model. We propose ISD-linUCB, a novel algorithm that learns such invariances from offline bandit data by decomposing the context-action feature space into two orthogonal subspaces: in one of the two subspaces, which has  dimension $p^{\operatorname{inv}}$, the relationship between reward and context-action features remains stationary. When sufficient offline data are available, the regret of ISD-linUCB scales with $(p-p^{\operatorname{inv}})$ rather than $p$, leading to a substantial performance gain in environments that are subject to rapid changes.

\subsection*{Acknowledgments}
NP and ML are supported by a research grant (0069071) from Novo Nordisk Fonden.

\newpage
\bibliographystyle{abbrvnat}
\bibliography{refs}

\begin{thebibliography}{21}
\providecommand{\natexlab}[1]{#1}
\providecommand{\url}[1]{\texttt{#1}}
\expandafter\ifx\csname urlstyle\endcsname\relax
  \providecommand{\doi}[1]{doi: #1}\else
  \providecommand{\doi}{doi: \begingroup \urlstyle{rm}\Url}\fi

\bibitem[Abbasi-Yadkori et~al.(2011)Abbasi-Yadkori, P{\'a}l, and
  Szepesv{\'a}ri]{abbasi2011improved}
Y.~Abbasi-Yadkori, D.~P{\'a}l, and C.~Szepesv{\'a}ri.
\newblock Improved algorithms for linear stochastic bandits.
\newblock \emph{Advances in Neural Information Processing Systems}, 24, 2011.

\bibitem[Agrawal and Goyal(2013)]{agrawal2013thompson}
S.~Agrawal and N.~Goyal.
\newblock Thompson sampling for contextual bandits with linear payoffs.
\newblock In \emph{International Conference on Machine Learning}, pages
  127--135. PMLR, 2013.

\bibitem[Auer(2002)]{auer2002using}
P.~Auer.
\newblock Using confidence bounds for exploitation-exploration trade-offs.
\newblock \emph{Journal of Machine Learning Research}, 3\penalty0
  (Nov):\penalty0 397--422, 2002.

\bibitem[Bilaj et~al.(2024)Bilaj, Dhouib, and Maghsudi]{bilaj2024meta}
S.~Bilaj, S.~Dhouib, and S.~Maghsudi.
\newblock Meta learning in bandits within shared affine subspaces.
\newblock In \emph{27th International Conference on Artificial Intelligence and
  Statistics}, pages 523--531. PMLR, 2024.

\bibitem[Cheung et~al.(2019)Cheung, Simchi-Levi, and Zhu]{cheung2019learning}
W.~C. Cheung, D.~Simchi-Levi, and R.~Zhu.
\newblock Learning to optimize under non-stationarity.
\newblock In \emph{22nd International Conference on Artificial Intelligence and
  Statistics}, pages 1079--1087. PMLR, 2019.

\bibitem[Chu et~al.(2011)Chu, Li, Reyzin, and Schapire]{chu2011contextual}
W.~Chu, L.~Li, L.~Reyzin, and R.~Schapire.
\newblock Contextual bandits with linear payoff functions.
\newblock In \emph{14th International Conference on Artificial Intelligence and
  Statistics}, pages 208--214, 2011.

\bibitem[Dani et~al.(2008)Dani, Hayes, and Kakade]{dani2008stochastic}
V.~Dani, T.~P. Hayes, and S.~M. Kakade.
\newblock Stochastic linear optimization under bandit feedback.
\newblock In \emph{21st Annual Conference on Learning Theory}, pages 355--366,
  2008.

\bibitem[Kausik et~al.(2024)Kausik, Tan, and Tewari]{kausik2024leveraging}
C.~Kausik, K.~Tan, and A.~Tewari.
\newblock Leveraging offline data in linear latent bandits.
\newblock \emph{arXiv preprint arXiv:2405.17324}, 2024.

\bibitem[Lattimore and Szepesv{\'a}ri(2020)]{lattimore2020bandit}
T.~Lattimore and C.~Szepesv{\'a}ri.
\newblock \emph{Bandit algorithms}.
\newblock Cambridge University Press, 2020.

\bibitem[Lazzaretto et~al.(2025)Lazzaretto, Peters, and
  Pfister]{lazzaretto2025invariant}
M.~Lazzaretto, J.~Peters, and N.~Pfister.
\newblock Invariant subspace decomposition.
\newblock \emph{Journal of Machine Learning Research}, 26\penalty0
  (95):\penalty0 1--56, 2025.

\bibitem[Li et~al.(2010)Li, Chu, Langford, and Schapire]{li2010contextual}
L.~Li, W.~Chu, J.~Langford, and R.~E. Schapire.
\newblock A contextual-bandit approach to personalized news article
  recommendation.
\newblock In \emph{Proceedings of the 19th international conference on World
  wide web}, pages 661--670, 2010.

\bibitem[Papini et~al.(2021)Papini, Tirinzoni, Restelli, Lazaric, and
  Pirotta]{papini2021leveraging}
M.~Papini, A.~Tirinzoni, M.~Restelli, A.~Lazaric, and M.~Pirotta.
\newblock Leveraging good representations in linear contextual bandits.
\newblock In \emph{38th International Conference on Machine Learning}, pages
  8371--8380. PMLR, 2021.

\bibitem[Qin et~al.(2022)Qin, Menara, Oymak, Ching, and
  Pasqualetti]{qin2022non}
Y.~Qin, T.~Menara, S.~Oymak, S.~Ching, and F.~Pasqualetti.
\newblock Non-stationary representation learning in sequential linear bandits.
\newblock \emph{IEEE Open Journal of Control Systems}, 1:\penalty0 41--56,
  2022.

\bibitem[Russac et~al.(2019)Russac, Vernade, and Capp{\'e}]{russac2019weighted}
Y.~Russac, C.~Vernade, and O.~Capp{\'e}.
\newblock Weighted linear bandits for non-stationary environments.
\newblock \emph{Advances in Neural Information Processing Systems}, 32, 2019.

\bibitem[Saengkyongam et~al.(2023)Saengkyongam, Thams, Peters, and
  Pfister]{saengkyongam2023invariant}
S.~Saengkyongam, N.~Thams, J.~Peters, and N.~Pfister.
\newblock Invariant policy learning: A causal perspective.
\newblock \emph{IEEE Transactions on Pattern Analysis and Machine
  Intelligence}, 45\penalty0 (7):\penalty0 8606--8620, 2023.

\bibitem[Stewart and Sun(1990)]{stewart1990matrixperturbation}
G.~W. Stewart and J.-g. Sun.
\newblock \emph{Matrix perturbation theory}.
\newblock Computer science and scientific computing. Academic Press, 1990.
\newblock URL \url{https://cir.nii.ac.jp/crid/1971149384759715084}.

\bibitem[Trella et~al.(2024)Trella, Dempsey, Doshi-Velez, and
  Murphy]{trella2024non}
A.~L. Trella, W.~Dempsey, F.~Doshi-Velez, and S.~A. Murphy.
\newblock Non-stationary latent auto-regressive bandits.
\newblock \emph{arXiv preprint arXiv:2402.03110}, 2024.

\bibitem[Tropp(2012)]{tropp2012user}
J.~A. Tropp.
\newblock User-friendly tail bounds for sums of random matrices.
\newblock \emph{Foundations of computational mathematics}, 12\penalty0
  (4):\penalty0 389--434, 2012.

\bibitem[Vershynin(2018)]{vershynin2018high}
R.~Vershynin.
\newblock \emph{High-dimensional probability: An introduction with applications
  in data science}, volume~47.
\newblock Cambridge University Press, 2018.

\bibitem[Yu et~al.(2015)Yu, Wang, and Samworth]{yu2015useful}
Y.~Yu, T.~Wang, and R.~J. Samworth.
\newblock A useful variant of the {D}avis--{K}ahan theorem for statisticians.
\newblock \emph{Biometrika}, 102\penalty0 (2):\penalty0 315--323, 2015.

\bibitem[Zhao et~al.(2020)Zhao, Zhang, Jiang, and Zhou]{zhao2020simple}
P.~Zhao, L.~Zhang, Y.~Jiang, and Z.-H. Zhou.
\newblock A simple approach for non-stationary linear bandits.
\newblock In \emph{23rd International Conference on Artificial Intelligence and
  Statistics}, pages 746--755. PMLR, 2020.

\end{thebibliography}
\newpage

\appendix

\section{Further related works}
\label{sec:related works}
\paragraph{Linear contextual bandits}
The stochastic linear bandit problem is extensively studied in the literature. It is commonly assumed \citep[see, for example,][]{lattimore2020bandit} that there exists an unknown linear parameter of dimension $p$ that is shared among all context-action features: this allows in particular to deal with arbitrarily large action spaces, and to obtain regret bounds that do not depend on the cardinality $K$ of the action space. Two standard algorithms proposed in the literature to solve this problem are LinUCB (linear upper confidence bound), and LinTS (linear Thompson sampling). In UCB-based algorithms, at all rounds the agent chooses the action maximizing an upper confidence bound on the estimated regret. \citet{li2010contextual} consider finite action spaces and build such confidence bound directly on the estimated regret for each possible action. \citet{chu2011contextual} show a regret bound for this version of LinUCB of $\tilde{O}(\sqrt{pKT})$. \citet{abbasi2011improved} propose a more general UCB algorithm that allows for infinitely large action spaces, by constructing a confidence set for the estimated linear parameter rather than for the reward, and show a regret bound of $\tilde{O}(p\sqrt{T})$. Linear bandits based on Thompson sampling \citep{agrawal2013thompson} at all rounds sample from the rewards posterior distribution for all possible arms, and then choose the action maximizing the sampled reward. LinTS is shown to achieve a regret of $\tilde{O}(p^{\frac{3}{2}}\sqrt{T})$. 

In these algorithms, no assumptions (besides boundedness of the context-action features and of the bandit parameter) are made on the nature of the observed contexts. \citet{papini2021leveraging} review different diversity conditions on the context-action features space
that allow to improve the regret bounds, leading to either logarithmic or constant regret in $T$. One example is when the variance of the optimal context-action features is a positive definite matrix, meaning that all directions in $\mathbb{R}^p$ are potentially optimal for some context.

\paragraph{Linear latent contextual bandits}

To deal with heterogeneity in the reward function, \citet{kausik2024leveraging} consider a latent linear contextual bandits framework, where a latent linear parameter $\theta\in\mathbb{R}^{p_l}$, with $p_l<p$, is shared across all heterogeneous observations. They assume the reward is of the form $R_t^{a_t} = \varphi(X_t, a_t)^{\top}\gamma_0 + \epsilon_t$, with $\gamma_0= U\theta$ and $U\in\mathbb{R}^{p\times p_l}$ is a unitary matrix. They propose to learn the matrix $U$ from offline data, and use this to obtain a regret bound that depends on the dimension $p_l$ of the latent parameter rather than the context dimension $p$. 

\citet{bilaj2024meta} assume that the agent sequentially interacts with different tasks such that, for each task, the parameter $\gamma_0$ is independently drawn from the same distribution. They assume there exists an orthogonal projection matrix $P$ with rank $p_l$ such that $\mathbb{E}_{\gamma_0}[\|(I-P)(\gamma_0-\mathbb{E}[\gamma_0])\|^2]\ll \mathbb{E}_{\gamma_0}[\|P(\gamma_0-\mathbb{E}[\gamma_0])\|^2]\le \mathbb{E}_{\gamma_0}[\|\gamma_0-\mathbb{E}[\gamma_0]\|^2]$, namely the variance of the distribution of $\gamma_0$ is very low along $p-p_l$ orthogonal directions. They show a regret bound that depends on $p_l$ rather than $p$.

\citet{qin2022non} consider a multi-task sequential linear bandit model where the agent plays a sequence of $S$ bandit tasks (drawn from $m$ different environments) for $N$ rounds each. They model the reward as $R_t^{a_t}=\varphi(X_t, a_t)^{\top}\gamma_{s(t)}+\epsilon_t$ and assume that there exists a shared latent parameter $\theta\in\mathbb{R}^{p_l}$ such that for all $s\in[S]$ there exists a matrix $U_s\in\mathbb{R}^{p\times p_l}$ such that $\gamma_s=U_s\theta$. They propose an algorithm that sequentially learns the latent representation and achieves a regret of $\tilde{O}(pp_l\sqrt{mSN} + p_lS\sqrt{N})$.

\citet{trella2024non} consider a non-stationary reward model governed by an underlying latent variable evolving as an AR process of order $\ell$. They show a regret bound of $\tilde{O}(\ell\sqrt{pT})$.

\paragraph{Non-stationary linear bandits}
\citet{cheung2019learning} and \citet{russac2019weighted} consider the linear bandit problem under non-stationarity of the environment, and in particular allow for the linear bandit parameter to change through time, so that the reward model is of the form $R_t^a = \varphi(X_t, a)^{\top}\gamma_{0,t} + \epsilon_t$.
They define the variation budget $B_T$ as a constant such that $\sum_{t=1}^{T-1} \|\gamma_{0,t+1}-\gamma_{0,t}\|_2 \le B_T$. \citet{cheung2019learning} propose an algorithm based on a sliding window regularized least squares estimator, \cite{russac2019weighted} propose instead to weight past observations by a discounting factor. Both achieve a regret bound of $\tilde{O}(p^{\frac{7}{8}}T^{\frac{3}{4}}B_T^{\frac{1}{4}})$. As shown by \citet{zhao2020simple}, this same regret is achieved by periodically restarting the standard LinUCB algorithm, where the optimal number of rounds to be played before each restart depends on $B_T$.
This means that the excess regret compared to the stationary case is due to applying the standard LinUCB algorithm assuming that the linear parameter is fixed, and incurring an additional loss that depends on how much the parameter is changing in the time period between two restarts. 
\citet{cheung2019learning} also show a lower bound for this setting of $\Omega(p^{\frac{2}{3}}T^{\frac{2}{3}}B_T^{\frac{1}{3}})$.

\section{Additional experiments}
\subsection{Projection error}
\label{sec:projection_err}
The key role of Assumption~\ref{ass:subspace_dec_error} is to ensure that the projection error on $\mathcal{S}^{\operatorname{inv}}$ and $\mathcal{S}^{\operatorname{res}}$ decreases with the sample size $T_0$ used to estimate the matrix $U$ (and thus the projection matrices onto the subspaces). Within the same experiment described in Section~\ref{sec:simulations_est_isd}, we evaluate such error, i.e., $\|\hat{\Pi}^{\mathcal{S}^{\operatorname{inv}}}-\Pi^{\mathcal{S}^{\operatorname{inv}}}\|_{\operatorname{op}}$, and show (Figure~\ref{fig:projection_err}) that it indeed decreases approximately as $\sqrt{T_0}$, as assumed in Assumption~\ref{ass:subspace_dec_error}.
\begin{figure}[ht]
    \centering
\includegraphics[width=0.6\linewidth]{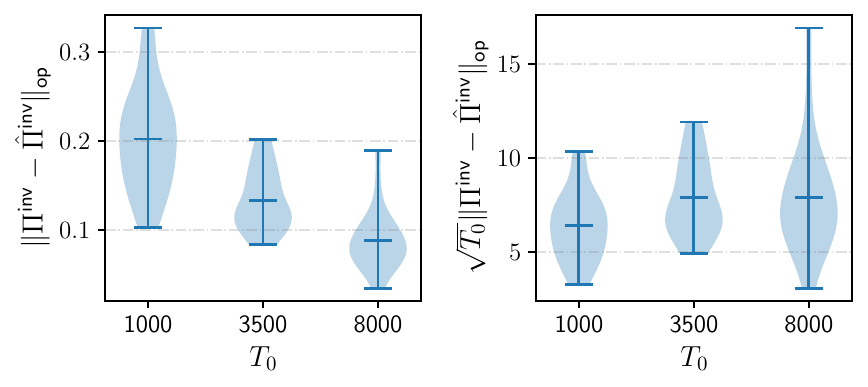}
    \caption{Projection error for $T_0\in\{1000, 3500, 8000\}$. The plot on the left shows the same values multiplied by $\sqrt{T_0}$, confirming our assumption.}
    \label{fig:projection_err}
\end{figure}

\subsection{Other non-stationary algorithms}
\label{sec:non-stationary_comp}
Within the same setting of the second experiment described in Section~\ref{sec:simulations_oracle_subs}, we add an additional comparison with two non-stationary algorithms: one using a sliding window (SW-linUCB) by \citet{cheung2019learning} and one using a discounting factor (D-linUCB) by \citet{russac2019weighted}. In this setting, the linear parameter $\gamma_{0,t}$ is fixed in the time horizon $[T]$, so the dimension of the window for SW-linUCB is set to $T$ and the discounting factor for D-linUCB is set to $0.999$, and Figure~\ref{fig:non-stat} shows that these two algorithms indeed have a comparable performance to LinUCB.
In particular, we keep $\gamma_{0,t}$ fixed because our focus is on the reduction of the dimensionality (for the non-stationary part of the problem). When this assumption fails, we could still integrate the ISD framework within existing non-stationary algorithms to improve their performance in terms of dimensionality, as we do in the presented experiments with LinUCB. 

\begin{figure}[h]
    \centering
    \includegraphics[width=0.8\linewidth]{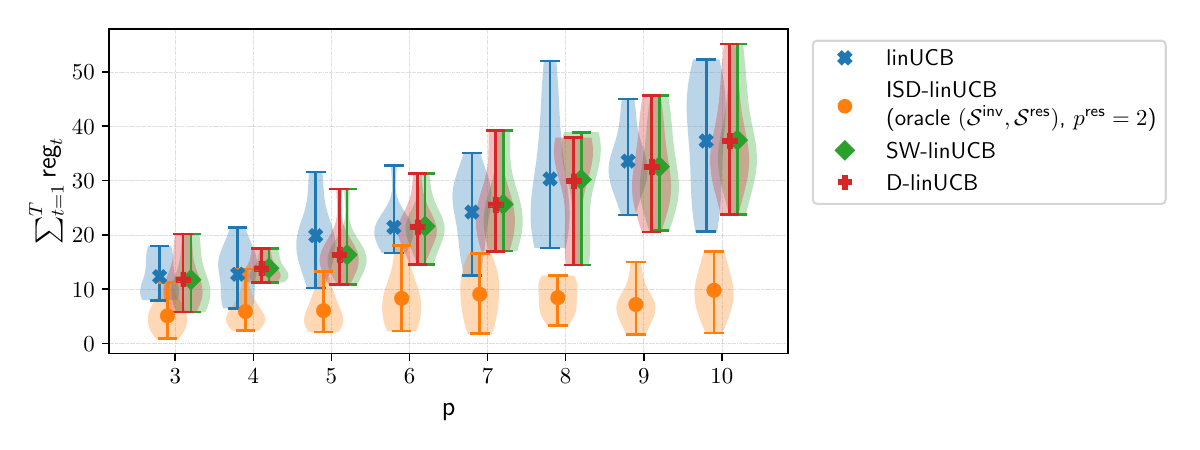}
    \caption{Cumulative regret of standard LinUCB, SW-linUCB, D-linUCB and ISD-linUCB with oracle $(\mathcal{S}^{\operatorname{inv}}, \mathcal{S}^{\operatorname{res}})$ for $T=100$ and increasing values of context-action feature dimension $p$. For ISD-linUCB, the invariant component $\beta^{\operatorname{inv}}$ is estimated using $T_0=2000$ observations.  $p^{\operatorname{inv}}$ varies from $3$ to $10$, while the $p^{\operatorname{res}}$ is fixed to $2$.}
    \label{fig:non-stat}
\end{figure}
The same happens for the simulation experiments presented in Section~\ref{sec:simulations_est_isd}, as shown in Figure~\ref{fig:non-stat2}.
\begin{figure}[ht]
    \centering
    \includegraphics[width=0.7\linewidth]{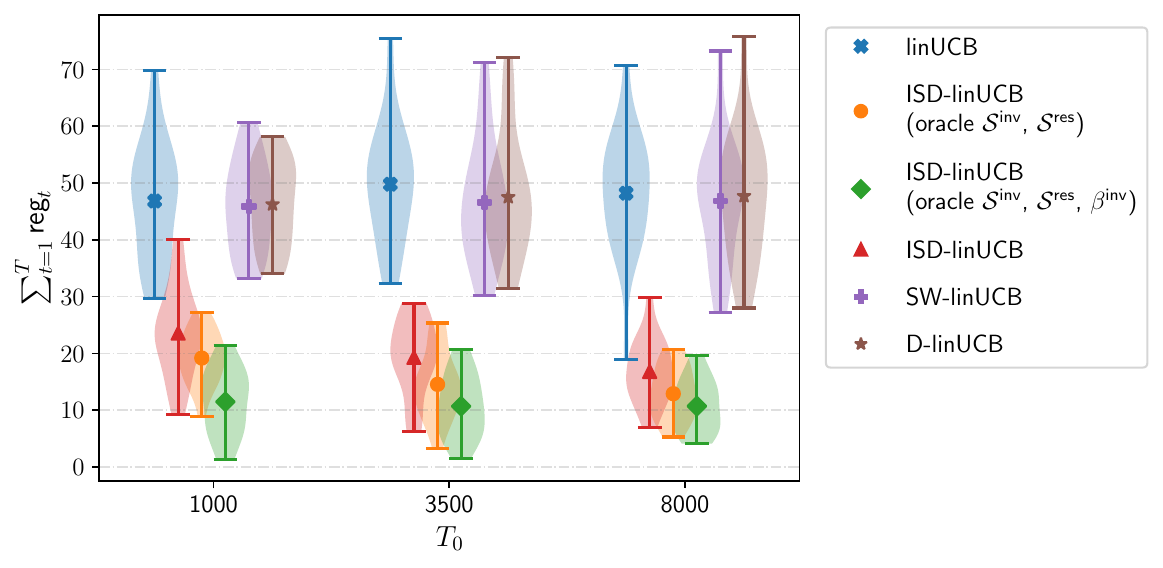}
    \caption{Cumulative regret for $T_0\in\{1000, 3500, 8000\}$ for ISD-linUCB, in comparison with the same algorithm having oracle information and with the standard linUCB alogorithm, with SW-linUCB and D-linUCB (unaffected by $T_0$). For each $T_0$ the experiment is repeated $20$ times.}
    \label{fig:non-stat2}
\end{figure}
\FloatBarrier

\section{Regret analysis: proofs}

\subsection{Oracle $\mathcal{S}^{\operatorname{inv}}, \mathcal{S}^{\operatorname{res}}, \beta^{\operatorname{inv}}$}
\label{app:oracle_bound}

\begin{proof}[Proof of Theorem~\ref{thm:oracle_regret}]
We start by proving that the estimated confidence set $\hat{\mathcal{C}}^{\delta}_t$ around $\hat{\delta}^{\operatorname{res}}_t$ contains $\delta^{\operatorname{res}}_t$ with high probability.

Let $\tilde{\Sigma}^{\operatorname{res}}_{t}$ and $\hat{\Sigma}_{t}$ be defined as in Section~\ref{sec:algorithm} and \eqref{eq:cov_matrix_t}, respectively, and let $\tilde{\delta}_t^{\operatorname{res}}\coloneqq \big(\lambda I_{p^{\operatorname{res}}} + (\mathbf{X}^{\operatorname{res}})^{\top}_{t}\mathbf{X}^{\operatorname{res}}_t \big)^{-1}(\mathbf{X}_t^{\operatorname{res}})^{\top}\mathbf{R}_t^{\operatorname{res}}$.
For all $t\in[T]$, we have that
\begin{align}
    \|\hat{\delta}^{\operatorname{res}}_t-\delta^{\operatorname{res}}_t\|_{\hat{\Sigma}_{t-1}}& = \|\tilde{\delta}^{\operatorname{res}}_t-(U^{\operatorname{res}})^{\top}\delta^{\operatorname{res}}_t\|_{\tilde{\Sigma}^{\operatorname{res}}_{t-1}}  \notag\\
    & = \|(\tilde{\Sigma}^{\operatorname{res}}_{t-1})^{-1}\sum_{\tau=1}^{t-1}(U^{\operatorname{res}})^{\top}\varphi(X_{\tau}, a_{\tau})(\varphi(X_{\tau}, a_{\tau})^{\top}\delta^{\operatorname{res}}_{\tau}+\epsilon_{\tau})-(U^{\operatorname{res}})^{\top}\delta^{\operatorname{res}}_t\|_{\tilde{\Sigma}^{\operatorname{res}}_{t-1}} \notag\\
    & = \|(\tilde{\Sigma}^{\operatorname{res}}_{t-1})^{-1}\sum_{\tau=1}^{t-1}(U^{\operatorname{res}})^{\top}\varphi(X_{\tau}, a_{\tau})\epsilon_{\tau}-\lambda(\tilde{\Sigma}^{\operatorname{res}}_{t-1})^{-1}(U^{\operatorname{res}})^{\top}\delta^{\operatorname{res}}_t\|_{\tilde{\Sigma}^{\operatorname{res}}_{t-1}} \notag\\
    & = \|\sum_{\tau=1}^{t-1}(U^{\operatorname{res}})^{\top}\varphi(X_{\tau}, a_{\tau})\epsilon_{\tau}-\lambda(U^{\operatorname{res}})^{\top}\delta^{\operatorname{res}}_t\|_{(\tilde{\Sigma}^{\operatorname{res}}_{t-1})^{-1}} \notag\\
    & \le \|\sum_{\tau=1}^{t-1}(U^{\operatorname{res}})^{\top}\varphi(X_{\tau}, a_{\tau})\epsilon_{\tau}\|_{(\tilde{\Sigma}^{\operatorname{res}}_{t-1})^{-1}}  + \|\lambda(U^{\operatorname{res}})^{\top}\delta^{\operatorname{res}}_t\|_{(\tilde{\Sigma}^{\operatorname{res}}_{t-1})^{-1}} \notag\\
    & \le \|\sum_{\tau=1}^{t-1}(U^{\operatorname{res}})^{\top}\varphi(X_{\tau}, a_{\tau})\epsilon_{\tau}\|_{(\tilde{\Sigma}^{\operatorname{res}}_{t-1})^{-1}}  + \sqrt{\lambda}\|\delta^{\operatorname{res}}_t\|_2 \label{eq:delta_res_oracle_estimation_error}
\end{align}
From Lemma~\ref{lemma:self_norm_bound}, we obtain that, for all $\eta\in(0,1)$, with probability at least $1-\eta$, 
\begin{align*}
     \|\hat{\delta}^{\operatorname{res}}_t-\delta^{\operatorname{res}}_t\|_{\hat{\Sigma}_{t-1}}& \le \sigma\sqrt{2\log\left(\frac{1}{\eta}\right)+\log\left(\frac{\operatorname{det}(\tilde{\Sigma}^{\operatorname{res}}_{t-1})}{\operatorname{det}(\lambda I_{p^{\operatorname{res}}})}\right)} + \sqrt{\lambda}\|\delta^{\operatorname{res}}_t\|_2
\end{align*}
Assuming that $L$ and $M$ where chosen sufficiently large in Algorithm~\ref{alg:isd_linucb} such that they indeed bound the context-action features and the linear parameter respectively, we get that, with probability at least $1-\eta$,
\begin{align*}
    \|\hat{\delta}^{\operatorname{res}}_t-\delta^{\operatorname{res}}_t\|_{\hat{\Sigma}_{t-1}}& \le \sigma\sqrt{2\log\left(\frac{1}{\eta}\right)+p^{\operatorname{res}}\log\left(1+\frac{tL^{2}}{\lambda p^{\operatorname{res}}}\right)} + \sqrt{\lambda}M \eqqcolon\sqrt{\hat{\rho}^{\operatorname{res}}_t(\eta, L, M)}.
\end{align*}
Assumption~\ref{ass:isd} allows us to estimate $\gamma_{0, t}$ by separately estimating $\beta^{\operatorname{inv}}$ and $\delta^{\operatorname{res}}_t$ on the invariant and residual subspace respectively. This implies that 
the full regret consists of the sum of regret on the invariant and residual subspace, as follows 
\begin{align}
    \operatorname{reg}_t & = (\varphi(X_t, a^*_t)-\varphi(X_t, a_t))^{\top} \gamma_{0,t}\notag \\
        & \le \varphi(X_{t},a_t)^{\top}(\bar{\gamma}_t - \gamma_{0,t}) \notag\\
    & = \underbrace{\varphi(X_{t},a_t)^{\top}(\bar{\beta}_t-\beta^{\operatorname{inv}})}_{\operatorname{reg}^{\operatorname{inv}}_t}  + \underbrace{\varphi(X_{t},a_t)^{\top}(\bar{\delta}_t-\delta^{\operatorname{res}}_t)}_{\operatorname{reg}^{\operatorname{res}}_t} \notag \\
    & = {\operatorname{reg}^{\operatorname{res}}_t},\label{eq:inst_regret_separation}
\end{align}
where $\bar{\gamma}_t = \argmax_{\gamma\in\hat{\mathcal{C}}^{\beta}\oplus\hat{\mathcal{C}}^{\delta}_t}\varphi(X_t, a_t)^{\top}\gamma$, 
$(\bar{\beta}_t, \bar{\delta}_t)\coloneqq\argmax_{\beta\in\hat{\mathcal{C}}^{\beta}, \delta\in\hat{\mathcal{C}}^{\delta}_t}\varphi(X_t, a_t)^{\top}(\beta+\delta)$ and we have that by construction $\bar{\gamma}_t = \bar{\beta}_t+\bar{\delta}_t$. 
The first inequality holds because, by definition, $\bar{\gamma}_t$ maximizes the upper confidence bound at time $t$. The last equality follows from the assumption that we have oracle knowledge of $\beta^{\operatorname{inv}}$, therefore $\hat{\mathcal{C}}^{\beta} = \{\beta^{\operatorname{inv}}\}$ and $\operatorname{reg}_t^{\operatorname{inv}}=0$ 
and we only need to consider the regret on the residual subspace. 
Then, for all $t\in[T]$, the residual instantaneous regret is such that
\begin{align}
    \operatorname{reg}^{\operatorname{res}}_{t} &= 
    \varphi(X_t, a_t)^{\top}(\bar{\delta}_t-\delta^{\operatorname{res}}_t) \notag\\
    & \le \|(U^{\operatorname{res}})^{\top}\varphi(X_t, a_t)\|_{(\tilde{\Sigma}^{\operatorname{res}}_{t-1})^{-1}}  \|(U^{\operatorname{res}})^{\top}(\bar{\delta}_t-\delta^{\operatorname{res}}_t)\|_{\tilde{\Sigma}^{\operatorname{res}}_{t-1}} \notag\\
    & \le 2\|(U^{\operatorname{res}})^{\top}\varphi(X_t, a_t)\|_{(\tilde{\Sigma}^{\operatorname{res}}_{t-1})^{-1}}\sqrt{\hat{\rho}^{\operatorname{res}}_t(\eta,L,M)}\label{eq:residual_regret_bound_norms}, 
\end{align}
 The first inequality follows from Cauchy-Schwarz. The last inequality is obtained by adding and subtracting $\hat{\delta}^{\operatorname{res}}_t$ inside the second norm and using that $\bar{\delta}_t\in\hat{\mathcal{C}}_t^{\delta}$ by definition and that, with probability at least $1-\eta$, $\delta^{\operatorname{res}}_t\in\hat{\mathcal{C}}_t^{\delta}$. 
Moreover, we have that for all $t\in[T]$, $\operatorname{reg}_t\le 2LM$ and for $\eta=1/T$, with $T>2$, that $\hat{\rho}^{\operatorname{res}}_t(\eta,L,M)\ge1$.
Together with \eqref{eq:residual_regret_bound_norms}, it implies that 
\begin{align*}
    \operatorname{reg}^{\operatorname{res}}_t & \le \min\{2LM,  2\sqrt{
    \hat{\rho}^{\operatorname{res}}_t(\eta, L, M)}\|(U^{\operatorname{res}})^{\top}\varphi(X_t, a_t)\|_{(\tilde{\Sigma}^{\operatorname{res}}_{t-1})^{-1}}\} \\
    & \le 2\max\{LM, 1\}\sqrt{\hat{\rho}^{\operatorname{res}}_t(\eta, L, M)}\min\{1, \|(U^{\operatorname{res}})^{\top}\varphi(X_t, a_t)\|_{(\tilde{\Sigma}^{\operatorname{res}}_{t-1})^{-1}}\} \\
    & \le C^{\operatorname{res}}\sqrt{\hat{\rho}^{\operatorname{res}}_t(\eta,L,M)} \min\{1, \|(U^{\operatorname{res}})^{\top}\varphi(X_t, a_t)\|_{(\tilde{\Sigma}^{\operatorname{res}}_{t-1})^{-1}}\} 
\end{align*}
where $C^{\operatorname{res}}\coloneqq 2\max\{LM, 1\}$. Using that, for all $x\ge0$, $\min\{1, x\}\le 2\log(1+x)$, we can upper bound the cumulative full regret as
\begin{align}
   \operatorname{Reg}_{T} & = \sum_{t\in[T]}\operatorname{reg}^{\operatorname{res}}_{T} \le \sqrt{T\sum_{t\in[T]}(\operatorname{reg}^{\operatorname{res}}_t)^2} \notag\\ & \le  C^{\operatorname{res}}\sqrt{T\rho^{\operatorname{res}}_T(\eta, L, M)\sum_{t\in[T]}\log(1 + \|(U^{\operatorname{res}})^{\top}\varphi(X_t, a_t)\|^2_{(\tilde{\Sigma}^{\operatorname{res}}_{t-1})^{-1}})}.
   \label{eq:cumreg_res_bouund}
\end{align}
As a final step, we use the following Lemma by~\citet{lattimore2020bandit}, with notation adapted to our problem.
\begin{lemma}[\citet{lattimore2020bandit}, Lemma 19.4]
\label{lem:bandits_alg_contexts_bound}
    Consider the same assumptions of Theorem~\ref{thm:oracle_regret}. Let $\lambda>0$ and $L<\infty$. Then,
    \begin{align*}
    \sum_{t\in[T]}\log(1 + \|(U^{\operatorname{res}})^{\top}\varphi(X_t, a_t)\|^2_{(\tilde{\Sigma}^{\operatorname{res}}_{t-1})^{-1}}) \le p^{\operatorname{res}}\log\left(1+\frac{TL^2}{\lambda p^{\operatorname{res}}}\right).
\end{align*}
\end{lemma}
Lemma~\ref{lem:bandits_alg_contexts_bound} implies that 
\begin{align*}
    \operatorname{Reg}_T \le C^{\operatorname{res}}\sqrt{T p^{\operatorname{res}}\log\left(1+\frac{TL^2}{\lambda p^{\operatorname{res}}}\right)}\left(\sigma\sqrt{2\log\left(\frac{1}{\eta}\right)+p^{\operatorname{res}}\log\left(1+\frac{TL^{2}}{\lambda p^{\operatorname{res}}}\right)} + \sqrt{\lambda}M\right).
\end{align*}
Finally, choosing $\eta=1/T$ implies that $\operatorname{Reg}_T$ is $\tilde{O}(p^{\operatorname{res}}\sqrt{T})$.    
\end{proof}

\subsection{Oracle subspaces}
\label{sec:proof_oracle_subs}

\begin{proof}[Proof of Theorem~\ref{thm:estimated_beta_inv}]
We start by showing that, under oracle knowledge of the subspaces, the radius $\sqrt{\hat{\rho}^{\operatorname{inv}}_{T_0}(\eta, L, M)}$ for the confidence set $\hat{\mathcal{C}}^{\beta}$ introduced in \eqref{eq:beta_confidence_set} can be defined to be $\tilde{O}(\sqrt{p^{\operatorname{inv}}}+ \sqrt{\frac{1}{\lambda_0}})$.
To do so, we consider the estimation error on $\hat{\beta}^{\operatorname{inv}}$ defined in \eqref{eq:beta_inv_est}, where $\hat{U}^{\operatorname{inv}}$ is replaced by $U^{\operatorname{inv}}$.  
\begin{align}
    & \quad\; \|\hat{\beta}^{\operatorname{inv}}-\beta\|_{\hat{\Sigma}_{[-T_0]}}\\
    & = \|U^{\operatorname{inv}}(\tilde{\Sigma}^{\operatorname{inv}}_{[-T_0]})^{-1} \sum_{t=-T_0}^{-1}(U^{\operatorname{inv}})^{\top}\varphi(X_t, a_t)R_t^{a_t} - U^{\operatorname{inv}}(U^{\operatorname{inv}})^{\top}\beta^{\operatorname{inv}}\|_{\hat{\Sigma}_{[-T_0]}}\notag\\
    & = \|(\tilde{\Sigma}^{\operatorname{inv}}_{[-T_0]})^{-1} \sum_{t=-T_0}^{-1}(U^{\operatorname{inv}})^{\top}\varphi(X_t, a_t)R_t^{a_t} - (U^{\operatorname{inv}})^{\top}\beta^{\operatorname{inv}}\|_{\tilde{\Sigma}^{\operatorname{inv}}_{[-T_0]}} \notag\\
    & = \|(\tilde{\Sigma}^{\operatorname{inv}}_{[-T_0]})^{-1} \sum_{t=-T_0}^{-1}(U^{\operatorname{inv}})^{\top}\varphi(X_t, a_t)(\varphi(X_{t},a_t)^{\top}(\beta^{\operatorname{inv}}+\delta^{\operatorname{res}}_t)+\epsilon_t) - (U^{\operatorname{inv}})^{\top}\beta^{\operatorname{inv}}\|_{\tilde{\Sigma}^{\operatorname{inv}}_{[-T_0]}} \notag\\
    & = \|(\tilde{\Sigma}^{\operatorname{inv}}_{[-T_0]})^{-1} \sum_{t=-T_0}^{-1}(U^{\operatorname{inv}})^{\top}\varphi(X_t, a_t)(\varphi(X_{t},a_t)^{\top}U^{\operatorname{res}}(U^{\operatorname{res}})^{\top}\delta^{\operatorname{res}}_t +\epsilon_t)\|_{\tilde{\Sigma}^{\operatorname{inv}}_{[-T_0]}} \notag\\
    & \le \|(\tilde{\Sigma}^{\operatorname{inv}}_{[-T_0]})^{-1} \sum_{t=-T_0}^{-1}(U^{\operatorname{inv}})^{\top}\varphi(X_t, a_t)\varphi(X_{t},a_t)^{\top}U^{\operatorname{res}}(U^{\operatorname{res}})^{\top}\delta^{\operatorname{res}}_t\|_{\tilde{\Sigma}^{\operatorname{inv}}_{[-T_0]}} \notag\\
    &\quad + \|(\tilde{\Sigma}^{\operatorname{inv}}_{[-T_0]})^{-1} \sum_{t=-T_0}^{-1}(U^{\operatorname{inv}})^{\top}\varphi(X_t, a_t)\epsilon_t\|_{\tilde{\Sigma}^{\operatorname{inv}}_{[-T_0]}}\notag \\
    & =  \| \sum_{t=-T_0}^{-1}(U^{\operatorname{inv}})^{\top}\varphi(X_t, a_t)\varphi(X_{t},a_t)^{\top}U^{\operatorname{res}}(U^{\operatorname{res}})^{\top}\delta^{\operatorname{res}}_t\|_{(\tilde{\Sigma}^{\operatorname{inv}}_{[-T_0]})^{-1}}  \label{eq:cross_cov_finite_sample}\\
    &\quad + \| (\tilde{\mathbf{X}}^{\operatorname{inv}}_{T_0} )^{\top}{\bm \epsilon}_{T_0}\|_{(\tilde{\Sigma}^{\operatorname{inv}}_{[-T_0]})^{-1}}\label{eq:betainv_subgauss_error_term}
\end{align}
where $\tilde{\Sigma}^{\operatorname{inv}}_{[-T_0]}\coloneqq (U^{\operatorname{inv}})^{\top}\hat{\Sigma}_{[-T_0]}U^{\operatorname{inv}}$,  $\tilde{\mathbf{X}}^{\operatorname{inv}}_{T_0}\coloneqq [(U^{\operatorname{inv}})^{\top}\varphi(X_{-T_0}, a_{-T_0}), \dots, (U^{\operatorname{inv}})^{\top}\varphi(X_{-1}, a_{-1})]^{\top}\in\mathbb{R}^{T_0\times p^{\operatorname{inv}}}$ and ${\bm \epsilon}_{T_0}\coloneqq [\epsilon_{-T_0}, \dots,\epsilon_{-1}]^{\top}\in\mathbb{R}^{T_0}$. 
For the term in \eqref{eq:cross_cov_finite_sample}, it holds that 
\begin{align}
    &  \| \sum_{t=-T_0}^{-1}(U^{\operatorname{inv}})^{\top}\varphi(X_t, a_t)\varphi(X_{t},a_t)^{\top}U^{\operatorname{res}}(U^{\operatorname{res}})^{\top}\delta^{\operatorname{res}}_t\|_{(\tilde{\Sigma}^{\operatorname{inv}}_{[-T_0]})^{-1}}\notag\\
    \le & \|(\tilde{\Sigma}^{\operatorname{inv}}_{[-T_0]})^{-\frac{1}{2}}\|_{\operatorname{op}} \|\sum_{t=-T_0}^{-1}(U^{\operatorname{inv}})^{\top}\varphi(X_t, a_t)\varphi(X_{t},a_t)^{\top}U^{\operatorname{res}}(U^{\operatorname{res}})^{\top}\delta^{\operatorname{res}}_t\|_2 \notag\\
    \le & \sqrt{\frac{1}{\lambda_0 T_0}} \| \sum_{t=-T_0}^{-1} \alpha_t \|_2
    \label{eq:beta_bound_delta_res_cross_cov_term}
\end{align}
where $\alpha_t\coloneqq(U^{\operatorname{inv}})^{\top}\varphi(X_t, a_t)\varphi(X_t, a_t)^{\top}U^{\operatorname{res}}(U^{\operatorname{res}})^{\top}\delta^{\operatorname{res}}_t \in\mathbb{R}^{p^{\operatorname{inv}}}$. We bound $\|\sum_{t=-T_0}^{-1}\alpha_t\|_2$ using the matrix Hoeffding inequality \citep[see, for example,][Theorem 1.3, which we report below]{tropp2012user}.
\begin{lemma}[Matrix Hoeffding inequality]
\label{lem:matrix_hoeff}
    Let $\{\mathbf{M}_k\}_{k\in[K]}$ a finite sequence of independent, random, self-adjoint matrices of dimension $p$ and let $\{\mathbf{A}_k\}_{k\in [K]}$ be a sequence of fixed self-adjoint matrices of dimension $p$. Assume that, for all $k\in[K]$, $\mathbb{E}[\mathbf{M}_k]=0$ and $\mathbf{M}_k^2\preceq \mathbf{A}_k$ almost surely. Then, for all $\xi>0$, it holds that
    \begin{equation*}
        \mathbb{P}\left(\lambda_{\max}\left(\sum_{k\in[K]}\mathbf{M}_k\right)\ge \xi\right)\le p \exp\left(-\frac{\xi^2}{8\sigma^2}\right)\quad \text{where} \quad \sigma^2\coloneqq \left\|\sum_{k\in[K]}\mathbf{A}_k^2\right\|.
    \end{equation*}
\end{lemma}
By definition of ISD, it holds that, for all $t\in[-T_0]$, $\mathbb{E}[\alpha_t]=0$. Using Cauchy-Schwarz, we also have that $\|\alpha_t\|_2 \le L^2M$. Moreover, we can construct the self-adjoint dilation \citep[see][Section 2.6]{tropp2012user} for $\alpha_t$ as 
\begin{equation*}
    \mathscr{S}(\alpha_t)\coloneqq \begin{bmatrix}
        0 & \alpha_t \\
        \alpha_t^{\top} & 0
    \end{bmatrix} \in\mathbb{R}^{(p^{\operatorname{inv}}+1)\times(p^{\operatorname{inv}}+1)},
\end{equation*}
which is such that $\|\alpha_t\|_2 = \|\mathscr{S}(\alpha_t)\|_{\operatorname{op}} = \lambda_{\max}(\mathscr{S}(\alpha_t))$ and 
\begin{equation*}
    \mathscr{S}(\alpha_t)^2\coloneqq \begin{bmatrix}
        \alpha_t\alpha_t^{\top} & 0 \\
        0 & \alpha_t^{\top}\alpha_t
    \end{bmatrix}.
\end{equation*}
Hoeffding inequality can be now applied to the sequence of self-adjoint matrices $(\mathscr{S}(\alpha_t))_{t\in[-T_0]}$. Let 
$\sigma_{\alpha}^2\coloneqq\|\sum_{t\in[-T_0]} \mathscr{S}(\alpha_t)^2\|_{\operatorname{op}} \le \sum_{t\in[-T_0]}\|\mathscr{S}(\alpha_t)\|^2_{\operatorname{op}}\le T_0(L^2M)^2$. Then, for all $\xi\ge 0$ it holds by Hoeffding's inequality that
\begin{equation*}
    \mathbb{P}\left(\lambda_{\max}\left(\sum_{t\in[-T_0]}\mathscr{S}(\alpha_t)\right)\ge\xi\right)\le (p^{\operatorname{inv}}+1)e^{-\frac{\xi^2}{8\sigma_{\alpha}^2}}.
\end{equation*}
Using the definition of $\mathscr{S}(\alpha_t)$ and defining $\eta\coloneqq(p^{\operatorname{inv}}+1)e^{-\frac{\xi^2}{8\sigma_{\alpha}^2}}$, the above is equivalent to 
\begin{equation*}
    \mathbb{P}\left(\|\sum_{t\in[-T_0]}\alpha_t\|_2\le 2L^2M\sqrt{2T_0\log\left(\frac{p^{\operatorname{inv}}+1}{\eta}\right)}\right)\ge 1-\eta.
\end{equation*}
Therefore, with probability at least $1-\eta$, $\sqrt{\frac{1}{\lambda_0 T_0}} \|\sum_{t=-T_0}^{-1}\alpha_t\|_2$ is $O\left( \sqrt{\frac{1}{\lambda_0 }\log\left(\frac{p^{\operatorname{inv}}}{\eta}\right)}  \right)$.

We now need to bound the term $\| (\tilde{\mathbf{X}}^{\operatorname{inv}}_{T_0} )^{\top}{\bm \epsilon}_{T_0}\|_{(\tilde{\Sigma}^{\operatorname{inv}}_{[-T_0]})^{-1}}$ in \eqref{eq:betainv_subgauss_error_term}.
By Assumption~\ref{ass:min_eig}, we have that $\hat{\Sigma}_{[-T_0]}\succeq \lambda_0T_0 I_p$, which implies that $\tilde{\Sigma}^{\operatorname{inv}}_{[-T_0]}\succeq \lambda_0T_0 I_{p^{\operatorname{inv}}}$ and therefore $\tilde{\Sigma}^{\operatorname{inv}}_{[-T_0]}\succeq \frac{1}{2}(\lambda_0T_0 I_{p^{\operatorname{inv}}}+\tilde{\Sigma}^{\operatorname{inv}}_{[-T_0]})$. This implies that, 
\begin{align*}
    \| (\tilde{\mathbf{X}}^{\operatorname{inv}}_{T_0} )^{\top}{\bm \epsilon}_{T_0}\|_{(\tilde{\Sigma}^{\operatorname{inv}}_{[-T_0]})^{-1}} \le \sqrt{2}\|(\tilde{\mathbf{X}}^{\operatorname{inv}}_{T_0} )^{\top}{\bm \epsilon}_{T_0}\|_{(\lambda_0T_0 I_{p^{\operatorname{inv}}}+\tilde{\Sigma}^{\operatorname{inv}}_{T_0})^{-1}}
\end{align*}
Lemma~\ref{lemma:self_norm_bound} implies that, for all $\eta\in(0,1)$, with probability at least $1-\eta$,
\begin{align*}
    \|(\tilde{\mathbf{X}}^{\operatorname{inv}}_{T_0} )^{\top}{\bm \epsilon}_{T_0}\|_{(\lambda_0T_0 I_{p^{\operatorname{inv}}}+\tilde{\Sigma}^{\operatorname{inv}}_{[-T_0]})^{-1}} \le \sigma\sqrt{\log\left(\frac{1}{\eta}\right)+\frac{1}{2}\log\left(\frac{\operatorname{det}(\tilde{\Sigma}^{\operatorname{inv}}_{[-T_0]})}{\operatorname{det}(\lambda_0T_0I_{p^{\operatorname{inv}}})}\right)} 
\end{align*}
By the arithmetic mean---geometric mean inequality, it holds that
\begin{align*}
    \operatorname{det}(\tilde{\Sigma}^{\operatorname{inv}}_{[-T_0]}) \le \left(\frac{1}{p^{\operatorname{inv}}}\operatorname{trace}(\tilde{\Sigma}^{\operatorname{inv}}_{[-T_0]})\right)^{p^{\operatorname{inv}}} \le \left(\frac{T_0 L^{2}}{p^{\operatorname{inv}}}\right)^{p^{\operatorname{inv}}}
\end{align*}
where the last inequality follows from
$ \operatorname{trace}(\tilde{\Sigma}^{\operatorname{inv}}_{[-T_0]}) = \operatorname{trace}((\tilde{\mathbf{X}}^{\operatorname{inv}}_{T_0})^{\top}\tilde{\mathbf{X}}^{\operatorname{inv}}_{T_0})\le T_0L^{2}$. Moreover, $\operatorname{det}(\lambda_0T_0I_{p^{\operatorname{inv}}}) = (\lambda_0T_0)^{p^{\operatorname{inv}}}$.
Therefore, with probability at least $1-\eta$,
\begin{align}
     \|(\tilde{\mathbf{X}}^{\operatorname{inv}}_{T_0} )^{\top}{\bm \epsilon}_{T_0}\|_{(\lambda_0T_0 I_{p^{\operatorname{inv}}}+\tilde{\Sigma}^{\operatorname{inv}}_{[-T_0]})^{-1}}  & \le \sigma\sqrt{\log\left(\frac{1}{\eta}\right)+\frac{p^{\operatorname{inv}}}{2}\log\left(\frac{T_0L^2}{p^{\operatorname{inv}}\lambda_0T_0}\right)}  \label{eq:or_subs_beta_err_eps}
\end{align}
and, with probability at least $1-2\eta$,
\begin{align}
    & \|\hat{\beta}^{\operatorname{inv}}-\beta\|_{\hat{\Sigma}_{[-T_0]}} \notag\\
     \le &\sigma\sqrt{2\log\left(\frac{1}{\eta}\right)+p^{\operatorname{inv}}\log\left(\frac{L^2}{p^{\operatorname{inv}}\lambda_0}\right)} +  2L^2M\sqrt{\frac{2}{\lambda_0}\log\left(\frac{p^{\operatorname{inv}}+1}{\eta}\right)}\eqqcolon \hat{\rho}^{\operatorname{inv}}_{T_0}(\eta, L, M)
    \label{eq:oracle_beta_radius}
\end{align}
Using Assumption~\ref{ass:min_eig}, we further obtain that, with probability at least $1-2\eta$,
\begin{align}
    \|\hat{\beta}^{\operatorname{inv}}-\beta\|_2 & \le \sqrt{\frac{\hat{\rho}^{\operatorname{inv}}_{T_0}(\eta, L, M)}{\lambda_0 T_0}}
    \label{eq:beta_inv_2norm_bound}
\end{align}
Consider now Algorithm~\ref{alg:isd_linucb}. For all $a\in\mathcal{A}$, let $\operatorname{UCB}_t^{\beta,\delta}(a)\coloneqq\max_{\beta\in\hat{\mathcal{C}}^{\beta}, \delta\in\hat{\mathcal{C}}^{\delta}_t}\varphi(X_t, a)^{\top}(\beta+\delta)$ and let $(\bar{\beta}_t, \bar{\delta}_t)\coloneqq \argmax_{\beta\in\hat{\mathcal{C}}^{\beta}, \delta\in\hat{\mathcal{C}}^{\delta}_t} \varphi(X_t, a_t)^{\top}(\beta+\delta)$
. Then, as in \eqref{eq:inst_regret_separation}, for all $t\in[T]$, the instantaneous regret is such that
$
    \operatorname{reg}_t \le \operatorname{reg}^{\operatorname{inv}}_t + \operatorname{reg}^{\operatorname{res}}_t.
$
 We upper bound the $\operatorname{reg}^{\operatorname{inv}}_t$ as follows
    \begin{align*}
         \operatorname{reg}^{\operatorname{inv}}_t & \coloneqq ((U^{\operatorname{inv}})^{\top}\varphi(X_t, a_t))^{\top}(U^{\operatorname{inv}})^{\top}(\bar{\beta}_t-\beta^{\operatorname{inv}}) \\
         & \le  \|(U^{\operatorname{inv}})^{\top}\varphi(X_t, a_t)\|_2 \|(U^{\operatorname{inv}})^{\top}(\bar{\beta}_t-\beta^{\operatorname{inv}})\|_2 \\
         & \le\frac{ 2 L}{\sqrt{\lambda_0 T_0}}\left(\sigma\sqrt{2\log\left(\frac{1}{\eta}\right)+p^{\operatorname{inv}}\log\left(\frac{L^2}{p^{\operatorname{inv}}\lambda_0}\right)} +  2L^2M\sqrt{\frac{2}{\lambda_0}\log\left(\frac{p^{\operatorname{inv}}+1}{\eta}\right)}\right)
    \end{align*}
where the first inequality holds with probability $1-2\eta$ and follows from the definition of $\hat{\mathcal{C}}^{\beta}$ and from Cauchy-Schwarz inequality. 

    For $\operatorname{reg}^{\operatorname{res}}_t$, the analysis from Theorem~\ref{thm:oracle_regret} remains valid, with an additional term to be considered due to using the estimated invariant component instead of the oracle one. In particular, following the same steps as in \eqref{eq:delta_res_oracle_estimation_error}, we have that, for all $t\in[T]$,
    \begin{align*}
        \|\hat{\delta}^{\operatorname{res}}_t-\delta^{\operatorname{res}}_t\|_{\hat{\Sigma}_{t-1}}& \le \|\sum_{\tau=1}^{t-1}(U^{\operatorname{res}})^{\top}\varphi(X_{\tau}, a_{\tau})\epsilon_{\tau}\|_{(\tilde{\Sigma}^{\operatorname{res}}_{t-1})^{-1}}  + \sqrt{\lambda}\|\delta^{\operatorname{res}}_t\|_2 \\
        & + \|\sum_{\tau=1}^{t-1}(U^{\operatorname{res}})^{\top}\varphi(X_{\tau}, a_{\tau})\varphi(X_{\tau}, a_{\tau})^{\top}(\beta^{\operatorname{inv}}-\hat{\beta}^{\operatorname{inv}})\|_{(\tilde{\Sigma}^{\operatorname{res}}_{t-1})^{-1}}.
    \end{align*}
 We analyzed the first two terms in Theorem~\ref{thm:oracle_regret}. For the last term, we have that 
 \begin{align*}
 & \|\sum_{\tau=1}^{t-1}(U^{\operatorname{res}})^{\top}\varphi(X_{\tau}, a_{\tau})\varphi(X_{\tau}, a_{\tau})^{\top}(\beta^{\operatorname{inv}}-\hat{\beta}^{\operatorname{inv}})\|_{(\tilde{\Sigma}^{\operatorname{res}}_{t-1})^{-1}} \\
 \le & \|(\tilde{\Sigma}^{\operatorname{res}}_{t-1})^{-\frac{1}{2}}\|_{\operatorname{op}} \|\sum_{\tau=1}^{t-1}(U^{\operatorname{res}})^{\top}\varphi(X_{\tau}, a_{\tau})\varphi(X_{\tau}, a_{\tau})^{\top}U^{\operatorname{inv}}\|_{\operatorname{op}}\|(U^{\operatorname{inv}})^{\top}(\beta^{\operatorname{inv}}-\hat{\beta}^{\operatorname{inv}})\|_2.
\end{align*}
Using matrix Hoeffding inequality (see Lemma~\ref{lem:matrix_hoeff}) and following the same steps used to bound \eqref{eq:cross_cov_finite_sample} that, for all $\eta\in(0,1)$, with probability at least $1-\eta$,
\begin{align*}
    \|\sum_{\tau=1}^{t-1}(U^{\operatorname{res}})^{\top}\varphi(X_{\tau}, a_{\tau})\varphi(X_{\tau}, a_{\tau})^{\top}U^{\operatorname{inv}}\|_{\operatorname{op}} \le 2L^{2}\sqrt{2(t-1)\log\left(\frac{p+1}{\eta}\right)}. 
\end{align*}
Therefore, with probability at least $1-3\eta$,
\begin{align*}
 & \|\sum_{\tau=1}^{t-1}(U^{\operatorname{res}})^{\top}\varphi(X_{\tau}, a_{\tau})\varphi(X_{\tau}, a_{\tau})^{\top}(\beta^{\operatorname{inv}}-\hat{\beta}^{\operatorname{inv}})\|_{(\tilde{\Sigma}^{\operatorname{res}}_{t-1})^{-1}} \\
 \le & \frac{1}{\sqrt{\lambda}} 2L^{2}\sqrt{2(t-1)\log\left(\frac{p+1}{\eta}\right)} \sqrt{\frac{\hat{\rho}^{\operatorname{inv}}_{T_0}(\eta, L, M)}{\lambda_0 T_0}}.
\end{align*}
The right-hand side of the inequality can be added to $\sqrt{\hat{\rho}^{\operatorname{res}}_t(\eta, L, M)}$ obtained in the proof of Theorem~\ref{thm:oracle_regret}, so that, with probability at least $1-4\eta$
\begin{align*}
     \sum_{t\in[T]}\operatorname{reg}^{\operatorname{res}}_t & \le C^{\operatorname{res}}\sqrt{T p^{\operatorname{res}}\log\left(1+\frac{TL^2}{\lambda p^{\operatorname{res}}}\right)}\\
    & \times \left(\sigma\sqrt{2\log\left(\frac{1}{\eta}\right)+p^{\operatorname{res}}\log\left(1+\frac{TL^{2}}{\lambda p^{\operatorname{res}}}\right)} + \sqrt{\lambda}M
     2L^{2}\sqrt{\frac{2T\hat{\rho}^{\operatorname{inv}}_{T_0}(\eta, L, M)}{\lambda\lambda_0 T_0}\log\left(\frac{p+1}{\eta}\right)}\right).
\end{align*}
Finally, this implies that the cumulative regret can be upper bounded, for all $\eta\in(0, \frac{1}{5})$ with probability at least $1-5\eta$ as
\begin{align*}
    \operatorname{Reg}_T &  = \sum_{t\in[T]}\operatorname{reg}_t = \sum_{t\in[T]}\operatorname{reg}^{\operatorname{inv}} + \sum_{t\in[T]}\operatorname{reg}^{\operatorname{res}} \\
    & \le\frac{ 2 LT}{\sqrt{\lambda_0 T_0}}\left(\sigma\sqrt{2\log\left(\frac{1}{\eta}\right)+p^{\operatorname{inv}}\log\left(\frac{L^2}{p^{\operatorname{inv}}\lambda_0}\right)} +  2L^2M\sqrt{\frac{2}{\lambda_0}\log\left(\frac{p^{\operatorname{inv}}+1}{\eta}\right)}\right) \\
    & +  C^{\operatorname{res}}\sqrt{T p^{\operatorname{res}}\log\left(1+\frac{TL^2}{\lambda p^{\operatorname{res}}}\right)}\\
    & \times\left(\sigma\sqrt{2\log\left(\frac{1}{\eta}\right)+p^{\operatorname{res}}\log\left(1+\frac{TL^{2}}{\lambda p^{\operatorname{res}}}\right)} + \sqrt{\lambda}M + 2L^{2}\sqrt{\frac{2T\hat{\rho}^{\operatorname{inv}}_{T_0}(\eta, L, M)}{\lambda\lambda_0 T_0}\log\left(\frac{p+1}{\eta}\right)}\right) 
\end{align*}
Choosing $\eta'=1/T$, where $\eta'\coloneqq5\eta$, implies that $\operatorname{Reg}_T$ is, with probability at least $1-\eta'$ $\tilde{O}\left(\sqrt{T}\left(p^{\operatorname{res}}+ \sqrt{\frac{p^{\operatorname{inv}}T}{\lambda_0 T_0}} + \frac{1}{\lambda_0}\sqrt{\frac{T}{T_0}}\right)\right)$.
\end{proof}

\subsection{Accounting for subspace decomposition errors}
\label{sec:proof_full_isd}
\begin{proof}[Proof of Lemma~\ref{lem:beta_inv_bound}] 
    We want to bound the $\hat{\Sigma}_{[-T_0]}$-norm of the estimation error for the invariant component in the estimated invariant subspace $\hat{\mathcal{S}}^{\operatorname{inv}}$, i.e., $\|\hat{\beta}^{\operatorname{inv}}-\hat{\Pi}^{\mathcal{S}^{\operatorname{inv}}}\beta^{\operatorname{inv}}\|_{\hat{\Sigma}_{[-T_0]}}$.
    Let %
    $\tilde{\Sigma}^{\operatorname{inv}}_{[-T_0]}\coloneqq (\hat{U}^{\operatorname{inv}})^{\top}\hat{\Sigma}_{[-T_0]}\hat{U}^{\operatorname{inv}}$, $\mathbf{X}_{T_0}\coloneqq[\varphi(X_{-T_0}, a_{-T_0}), \dots, \varphi(X_{-1}, a_{-1})]^{\top}\in\mathbb{R}^{T_0\times p}$, $\mathbf{R}_{T_0}\coloneqq [R^{a_{-T_0}}_{-T_0}, \dots, R^{a_{-1}}_{-1}]^{\top}\in\mathbb{R}^{T_0}$ and $\bm{\epsilon}_{T_0}\coloneqq[\epsilon_{-T_0}, \dots, \epsilon_{-1}]^{\top}\in\mathbb{R}^{T_0}$. Then, we have that
    \begin{align}
         \|\hat{\beta}^{\operatorname{inv}} - \hat{\Pi}^{\mathcal{S}^{\operatorname{inv}}}\beta^{\operatorname{inv}}\|_{\hat{\Sigma}_{[-T_0]}} & = \|\hat{U}^{\operatorname{inv}}((\tilde{\Sigma}^{\operatorname{inv}}_{[-T_0]})^{-1} \sum_{t=-T_0}^{-1}(\hat{U}^{\operatorname{inv}})^{\top}\varphi(X_t, a_t)R_t^{a_t} -(\hat{U}^{\operatorname{inv}})^{\top}\beta^{\operatorname{inv}})\|_{\hat{\Sigma}_{[-T_0]}} \notag\\
         & \le \|(\tilde{\Sigma}^{\operatorname{inv}}_{[-T_0]})^{-1} (\hat{U}^{\operatorname{inv}})^{\top}\mathbf{X}_{T_0}^{\top}\mathbf{X}_{T_0}\beta^{\operatorname{inv}}- (\hat{U}^{\operatorname{inv}})^{\top}\beta^{\operatorname{inv}}\|_{\tilde{\Sigma}^{\operatorname{inv}}_{[-T_0]}} \notag\\
         & + \|(\tilde{\Sigma}^{\operatorname{inv}}_{[-T_0]})^{-1}\sum_{t=-T_0}^{-1}(\hat{U}^{\operatorname{inv}})^{\top}\varphi(X_t, a_t)\varphi(X_t, a_t)^{\top}\delta^{\operatorname{res}}_t\|_{\tilde{\Sigma}^{\operatorname{inv}}_{[-T_0]}} \notag\\
         & + \|(\tilde{\Sigma}^{\operatorname{inv}}_{[-T_0]})^{-1}(\hat{U}^{\operatorname{inv}})^{\top}\mathbf{X}_{T_0}^{\top}\bm{\epsilon}_{T_0}\|_{\tilde{\Sigma}^{\operatorname{inv}}_{[-T_0]}} \notag\\
         & =  \|(\hat{U}^{\operatorname{inv}})^{\top}\mathbf{X}_{T_0}^{\top}\mathbf{X}_{T_0}(U^{\operatorname{inv}}(U^{\operatorname{inv}})^{\top}- \hat{U}^{\operatorname{inv}}(\hat{U}^{\operatorname{inv}})^{\top})\beta^{\operatorname{inv}}\|_{(\tilde{\Sigma}^{\operatorname{inv}}_{[-T_0]})^{-1}} \label{eq:beta_err_beta_term}\\
         & + \|\sum_{t=-T_0}^{-1}(\hat{U}^{\operatorname{inv}})^{\top}\varphi(X_t, a_t)\varphi(X_t, a_t)^{\top}\delta^{\operatorname{res}}_t\|_{(\tilde{\Sigma}^{\operatorname{inv}}_{[-T_0]})^{-1}} \label{eq:beta_err_delta_term}\\
         & + \|(\hat{U}^{\operatorname{inv}})^{\top}\mathbf{X}_{T_0}^{\top}\bm{\epsilon}_{T_0}\|_{(\tilde{\Sigma}^{\operatorname{inv}}_{[-T_0]})^{-1}} \label{eq:beta_err_eps_term}
    \end{align}
The term in \eqref{eq:beta_err_eps_term} can be upper bounded exactly as in Section~\ref{sec:proof_oracle_subs} (see Equation~\eqref{eq:or_subs_beta_err_eps}), since the presence of $\hat{U}^{\operatorname{inv}}$ in place of $U^{\operatorname{inv}}$ does not influence its analysis.
The terms in \eqref{eq:beta_err_beta_term} and \eqref{eq:beta_err_delta_term} appear because of the misalignment between the true and the estimated invariant subspace: the former represents how much of the invariant parameter we are not able to estimate due to the difference between $\hat{\mathcal{S}}^{\operatorname{inv}}$ and $\mathcal{S}^{\operatorname{inv}}$, the latter quantifies how much of the true residual parameter enters in the estimation of the invariant component due to the intersection between $\hat{\mathcal{S}}^{\operatorname{inv}}$ and $\mathcal{S}^{\operatorname{res}}$. For \eqref{eq:beta_err_beta_term}, we have that
\begin{align*}
    & \|(\hat{U}^{\operatorname{inv}})^{\top}\mathbf{X}_{T_0}^{\top}\mathbf{X}_{T_0}(\Pi^{\mathcal{S}^{\operatorname{inv}}}- \hat{\Pi}^{\mathcal{S}^{\operatorname{inv}}})\beta^{\operatorname{inv}}\|_{(\tilde{\Sigma}^{\operatorname{inv}}_{[-T_0]})^{-1}} \\
      =  & \|(\tilde{\Sigma}^{\operatorname{inv}}_{[-T_0]})^{-\frac{1}{2}}(\hat{U}^{\operatorname{inv}})^{\top}\mathbf{X}_{T_0}^{\top}\mathbf{X}_{T_0}(\Pi^{\mathcal{S}^{\operatorname{inv}}}- \hat{\Pi}^{\mathcal{S}^{\operatorname{inv}}})\beta^{\operatorname{inv}}\|_{2} \\
       \le &\sqrt{\sum_{t\in[-T_0]} \|(\tilde{\Sigma}^{\operatorname{inv}}_{[-T_0]})^{-\frac{1}{2}}(\hat{U}^{\operatorname{inv}})^{\top}\varphi(X_t, a_t) \|^2_2 } \sqrt{\sum_{t\in[-T_0]}(\varphi(X_t, a_t)^{\top}(\Pi^{\mathcal{S}^{\operatorname{inv}}}- \hat{\Pi}^{\mathcal{S}^{\operatorname{inv}}})\beta^{\operatorname{inv}})^2}\\
       \le & \sqrt{p^{\operatorname{inv}}}\sqrt{T_0 L^2M^2 (\Delta\Pi)^2} \eqqcolon \hat{\rho}^{\operatorname{inv}}_1\\ 
\end{align*}
    The first inequality uses Cauchy-Schwarz inequality. The second inequality follows from the fact that
    \begin{align*}
        & \quad\; \sum_{t\in[-T_0]} \|(\tilde{\Sigma}^{\operatorname{inv}}_{[-T_0]})^{-\frac{1}{2}}(\hat{U}^{\operatorname{inv}})^{\top}\varphi(X_t, a_t) \|^2_2 \\ & =  \sum_{t\in[-T_0]} \varphi(X_t, a_t)^{\top}\hat{U}^{\operatorname{inv}}(\tilde{\Sigma}^{\operatorname{inv}}_{[-T_0]})^{-1}(\hat{U}^{\operatorname{inv}})^{\top}\varphi(X_t, a_t)\\
        & = \operatorname{trace}\left(\sum_{t\in[-T_0]} (\hat{U}^{\operatorname{inv}})^{\top}\varphi(X_t, a_t)  \varphi(X_t, a_t)^{\top}\hat{U}^{\operatorname{inv}}(\tilde{\Sigma}^{\operatorname{inv}}_{[-T_0]})^{-1}\right) \\
        & = \operatorname{trace}\left(I_{p^{\operatorname{inv}}}\right)\\
        & = p^{\operatorname{inv}},
    \end{align*}
and by singling out the norm of the individual factors under the second square root. Finally, under Assumption~\ref{ass:subspace_dec_error} we obtain that, for all $\eta\in(0,1)$, it holds with probability at least $1-\eta$ that $\|(\hat{U}^{\operatorname{inv}})^{\top}\mathbf{X}_{T_0}^{\top}\mathbf{X}_{T_0}(\Pi^{\mathcal{S}^{\operatorname{inv}}}- \hat{\Pi}^{\mathcal{S}^{\operatorname{inv}}})\beta^{\operatorname{inv}}\|_{(\tilde{\Sigma}^{\operatorname{inv}}_{[-T_0]})^{-1}} $ is $O(\sqrt{p^{\operatorname{inv}}\log(\frac{p}{\eta})})$. 

We now need to upper bound the term in \eqref{eq:beta_err_delta_term}, that is 
\begin{align*}
     &\|\sum_{t=-T_0}^{-1}(\hat{U}^{\operatorname{inv}})^{\top}\varphi(X_t, a_t)\varphi(X_t, a_t)^{\top}\delta^{\operatorname{res}}_t\|_{(\tilde{\Sigma}^{\operatorname{inv}}_{[-T_0]})^{-1}} \\
     =& \|\sum_{t=-T_0}^{-1}(\hat{U}^{\operatorname{inv}})^{\top}(U^{\operatorname{inv}}(U^{\operatorname{inv}})^{\top}+U^{\operatorname{res}}(U^{\operatorname{res}})^{\top})\varphi(X_t, a_t)\varphi(X_t, a_t)^{\top}U^{\operatorname{res}}(U^{\operatorname{res}})^{\top}\delta^{\operatorname{res}}_t\|_{(\tilde{\Sigma}^{\operatorname{inv}}_{[-T_0]})^{-1}} \\
      \le & \|\sum_{t=-T_0}^{-1}(\hat{U}^{\operatorname{inv}})^{\top}U^{\operatorname{inv}}(U^{\operatorname{inv}})^{\top}\varphi(X_t, a_t)\varphi(X_t, a_t)^{\top}U^{\operatorname{res}}(U^{\operatorname{res}})^{\top}\delta^{\operatorname{res}}_t\|_{(\tilde{\Sigma}^{\operatorname{inv}}_{[-T_0]})^{-1}}\\
      + & \|\sum_{t=-T_0}^{-1}(\hat{U}^{\operatorname{inv}})^{\top}U^{\operatorname{res}}(U^{\operatorname{res}})^{\top}\varphi(X_t, a_t)\varphi(X_t, a_t)^{\top}U^{\operatorname{res}}(U^{\operatorname{res}})^{\top}\delta^{\operatorname{res}}_t\|_{(\tilde{\Sigma}^{\operatorname{inv}}_{[-T_0]})^{-1}}
\end{align*}
 For the second term, it holds that
\begin{align*}
    & \|\sum_{t=-T_0}^{-1}(\hat{U}^{\operatorname{inv}})^{\top}U^{\operatorname{res}}(U^{\operatorname{res}})^{\top}\varphi(X_t, a_t)\varphi(X_t, a_t)^{\top}U^{\operatorname{res}}(U^{\operatorname{res}})^{\top}\delta^{\operatorname{res}}_t\|_{(\tilde{\Sigma}^{\operatorname{inv}}_{[-T_0]})^{-1}} \\
    \le  & \|(\tilde{\Sigma}^{\operatorname{inv}}_{[-T_0]})^{-\frac{1}{2}}\|_{\operatorname{op}}\|(\hat{U}^{\operatorname{inv}})^{\top}U^{\operatorname{res}}\|_{\operatorname{op}} T_0L^2M \\
     \le & \sqrt{\frac{1}{\lambda_0 T_0}} \Delta\Pi  T_0L^2M \\
     = & \sqrt{\frac{T_0}{\lambda_0}} \Delta\Pi L^2M \eqqcolon \hat{\rho}^{\operatorname{inv}}_2.
\end{align*}
Under Assumption~\ref{ass:subspace_dec_error}, for all $\eta\in(0, 1)$, this quantity is, with probability at least $1-\eta$, $O\left(\sqrt{\frac{\log(p/\eta)}{\lambda_0}}\right)$. For the first term, we obtain 
\begin{align*}
     & \|\sum_{t=-T_0}^{-1}(\hat{U}^{\operatorname{inv}})^{\top}U^{\operatorname{inv}}(U^{\operatorname{inv}})^{\top}\varphi(X_t, a_t)\varphi(X_t, a_t)^{\top}U^{\operatorname{res}}(U^{\operatorname{res}})^{\top}\delta^{\operatorname{res}}_t\|_{(\tilde{\Sigma}^{\operatorname{inv}}_{[-T_0]})^{-1}} \\
     \le & \|(\tilde{\Sigma}^{\operatorname{inv}}_{[-T_0]})^{-\frac{1}{2}}\|_{\operatorname{op}}\|(\hat{U}^{\operatorname{inv}})^{\top}U^{\operatorname{inv}}\|_{\operatorname{op}} \|\sum_{t=-T_0}^{-1}(U^{\operatorname{inv}})^{\top}\varphi(X_t, a_t)\varphi(X_t, a_t)^{\top}U^{\operatorname{res}}(U^{\operatorname{res}})^{\top}\delta^{\operatorname{res}}_t\|_2 \\
     \le & \sqrt{\frac{1}{\lambda_0 T_0}} \|\sum_{t=-T_0}^{-1}\alpha_t\|_2
\end{align*}
where $\alpha_t\coloneqq(U^{\operatorname{inv}})^{\top}\varphi(X_t, a_t)\varphi(X_t, a_t)^{\top}U^{\operatorname{res}}(U^{\operatorname{res}})^{\top}\delta^{\operatorname{res}}_t \in\mathbb{R}^{p^{\operatorname{inv}}}$ and in the last inequality we have used that $\|(\hat{U}^{\operatorname{inv}})^{\top}U^{\operatorname{inv}}\|_{\operatorname{op}}\le 1$ since both matrices have orthonormal columns. This is the same quantity that appears in \eqref{eq:beta_bound_delta_res_cross_cov_term} in the oracle subspaces case (together with \eqref{eq:beta_err_eps_term}, this forms the full oracle bound given in \eqref{eq:oracle_beta_radius}, which we denote here by $\hat{\rho}^{\operatorname{inv}}_3$). We have shown \eqref{eq:beta_bound_delta_res_cross_cov_term} to be, for all $\eta\in(0,1)$, with probability at least $1-\eta$,
$O\left( \sqrt{\frac{1}{\lambda_0 }\log\left(\frac{p^{\operatorname{inv}}}{\eta}\right)}  \right)$.  
For the same choice of $\eta$, this term is dominated by $O\left( \sqrt{\frac{1}{\lambda_0 }\log\left(\frac{p}{\eta}\right)} \right)$. Hence, the expression in \eqref{eq:beta_err_delta_term} is, with probability at least $1-2\eta$, $O\left( \sqrt{\frac{1}{\lambda_0 }\log\left(\frac{p^{\operatorname{inv}}}{\eta}\right)}  \right)$. 

Let
\begin{align}
    \hat{\rho}^{\operatorname{inv}}_{T_0}(\eta, L, M) \coloneqq \sum_{i=1}^3\hat{\rho}^{\operatorname{inv}}_{i} & = \sigma\sqrt{2\log\left(\frac{1}{\eta}\right)+p^{\operatorname{inv}}\log\left(\frac{L^2}{p^{\operatorname{inv}}\lambda_0}\right)} +  2L^2M\sqrt{\frac{2}{\lambda_0}\log\left(\frac{p^{\operatorname{inv}}+1}{\eta}\right)} \label{eq:beta_inv_radius_full}\\
    & + \sqrt{p^{\operatorname{inv}} T_0 }\Delta\Pi LM + \sqrt{\frac{T_0}{\lambda_0}} \Delta\Pi L^2M. \notag
\end{align}

By the union bound we have that, with probability at least $1-4\eta$, 
\begin{equation*}
    \|\hat{\beta}^{\operatorname{inv}} - \hat{\Pi}^{\mathcal{S}^{\operatorname{inv}}}\beta^{\operatorname{inv}}\|_{\hat{\Sigma}_{[-T_0]}} \le \sum_{i=1}^3\hat{\rho}^{\operatorname{inv}}_{i} \eqqcolon \hat{\rho}^{\operatorname{inv}}_{T_0}(\eta, L, M)
\end{equation*}$ $ is 
\begin{equation*}
    O\left(\sqrt{\left(\log\left(\frac{1}{\eta}\right)+p^{\operatorname{inv}}\log\left(\frac{1}{p^{\operatorname{inv}}\lambda_0}\right)\right)}\right) + O\left(\sqrt{p^{\operatorname{inv}}\log\left(\frac{p}{\eta}\right)}\right) +  O\left(\sqrt{\frac{1}{\lambda_0}\log\left(\frac{p}{\eta}\right)}\right).
\end{equation*}
This concludes the proof of the lemma. 
\end{proof}

\begin{proof}[Proof of Lemma~\ref{lem:delta_res_bound}]
    We want to bound the estimation error for the residual component in the residual subspace, that is, $\|\hat{\delta}^{\operatorname{res}}_t-\hat{\Pi}^{\mathcal{S}^{\operatorname{res}}}\delta^{\operatorname{res}}_{t}\|_{\hat{\Sigma}_{t-1}}$. Recall that we have defined $\tilde{\Sigma}^{\operatorname{res}}_{t}\coloneqq (\hat{U}^{\operatorname{res}})^{\top}\hat{\Sigma}_{t}\hat{U}^{\operatorname{res}}$. We start by stating explicitly the expression for $\hat{\delta}^{\operatorname{res}}_t$, obtaining that
    \begin{align}
        & \|\hat{\delta}^{\operatorname{res}}_t-\hat{\Pi}^{\mathcal{S}^{\operatorname{res}}}\delta^{\operatorname{res}}_t\|_{\hat{\Sigma}_{t-1}}\notag\\
         =& \|\hat{U}^{\operatorname{res}}(\tilde{\Sigma}^{\operatorname{res}}_{t-1})^{-1}\sum_{\tau\in[t-1]}(\hat{U}^{\operatorname{res}})^{\top}\varphi(X_{\tau}, a_{\tau})(R^{a_{\tau}}_{\tau}- \varphi(X_{\tau}, a_{\tau})^{\top}\hat{\beta}^{\operatorname{inv}})-\hat{U}^{\operatorname{res}}(\hat{U}^{\operatorname{res}})^{\top}\delta^{\operatorname{res}}_t\|_{\hat{\Sigma}_{t-1}} \notag\\
         = & \|(\tilde{\Sigma}^{\operatorname{res}}_{t-1})^{-1}\sum_{\tau\in[t-1]}(\hat{U}^{\operatorname{res}})^{\top}\varphi(X_{\tau}, a_{\tau})(R^{a_{\tau}}_{\tau}- \varphi(X_{\tau}, a_{\tau})^{\top}\hat{\beta}^{\operatorname{inv}})-(\hat{U}^{\operatorname{res}})^{\top}\delta^{\operatorname{res}}_t\|_{\tilde{\Sigma}^{\operatorname{res}}_{t-1}} \notag\\
         \le &  \|\sum_{\tau\in[t-1]}(\hat{U}^{\operatorname{res}})^{\top}\varphi(X_{\tau}, a_{\tau})\varphi(X_{\tau}, a_{\tau})^{\top}(U^{\operatorname{res}}(U^{\operatorname{res}})^{\top}-\hat{U}^{\operatorname{res}}(\hat{U}^{\operatorname{res}})^{\top})\delta^{\operatorname{res}}_{t} - \lambda (\hat{U}^{\operatorname{res}})^{\top}\delta^{\operatorname{res}}_{t}\|_{(\tilde{\Sigma}^{\operatorname{res}}_{t-1})^{-1}} \notag\\
         +& \|\sum_{\tau\in[t-1]}(\hat{U}^{\operatorname{res}})^{\top}\varphi(X_{\tau}, a_{\tau})\varphi(X_{\tau}, a_{\tau})^{\top}(\beta^{\operatorname{inv}}-\hat{\beta}^{\operatorname{inv}})\|_{(\tilde{\Sigma}^{\operatorname{res}}_{t-1})^{-1}} \notag\\
         + & \|\sum_{\tau\in[t-1]}(\hat{U}^{\operatorname{res}})^{\top}\varphi(X_{\tau}, a_{\tau})\epsilon_{\tau}\|_{(\tilde{\Sigma}^{\operatorname{res}}_{t-1})^{-1}}\notag\\
         \le &  \|\sum_{\tau\in[t-1]}(\hat{U}^{\operatorname{res}})^{\top}\varphi(X_{\tau}, a_{\tau})\varphi(X_{\tau}, a_{\tau})^{\top}(U^{\operatorname{res}}(U^{\operatorname{res}})^{\top}-\hat{U}^{\operatorname{res}}(\hat{U}^{\operatorname{res}})^{\top})\delta^{\operatorname{res}}_{t}\|_{(\tilde{\Sigma}^{\operatorname{res}}_{t-1})^{-1}} \label{eq:delta_bound_delta_term}\\
         +& \|\sum_{\tau\in[t-1]}(\hat{U}^{\operatorname{res}})^{\top}\varphi(X_{\tau}, a_{\tau})\varphi(X_{\tau}, a_{\tau})^{\top}(\beta^{\operatorname{inv}}-\hat{\beta}^{\operatorname{inv}})\|_{(\tilde{\Sigma}^{\operatorname{res}}_{t-1})^{-1}} \label{eq:delta_bound_beta_term}\\
         + & \|\sum_{\tau\in[t-1]}(\hat{U}^{\operatorname{res}})^{\top}\varphi(X_{\tau}, a_{\tau})\epsilon_{\tau}\|_{(\tilde{\Sigma}^{\operatorname{res}}_{t-1})^{-1}} + \sqrt{\lambda}\|\delta^{\operatorname{res}}_t\|_2\label{eq:delta_bound_eps_term}
    \end{align}
The analysis of the terms in \eqref{eq:delta_bound_eps_term} follows the same steps as the one for \eqref{eq:delta_res_oracle_estimation_error} in the oracle case. In particular, from Theorem~\ref{thm:oracle_regret} we have that for all $\eta\in(0,1)$, with probability at least $1-\eta$,%
\begin{align*}
    &\|\sum_{\tau\in[t-1]}(\hat{U}^{\operatorname{res}})^{\top}\varphi(X_{\tau}, a_{\tau})\epsilon_{\tau}\|_{(\tilde{\Sigma}^{\operatorname{res}}_{t-1})^{-1}} + \sqrt{\lambda}\|\delta^{\operatorname{res}}_t\|_2\\ \le &\sigma\sqrt{2\log\left(\frac{1}{\eta}\right)+p^{\operatorname{res}}\log\left(1+\frac{tL^{2}}{\lambda p^{\operatorname{res}}}\right)} + \sqrt{\lambda}M \eqqcolon \hat{\rho}^{\operatorname{res}}_1.
\end{align*}
For the term in \eqref{eq:delta_bound_delta_term}, we have that
\begin{align}
    & \|\sum_{\tau\in[t-1]}(\hat{U}^{\operatorname{res}})^{\top}\varphi(X_{\tau}, a_{\tau})\varphi(X_{\tau}, a_{\tau})^{\top}(\Pi^{\mathcal{S}^{\operatorname{res}}}-\hat{\Pi}^{\mathcal{S}^{\operatorname{res}}})\delta^{\operatorname{res}}_{t}\|_{(\tilde{\Sigma}^{\operatorname{res}}_{t-1})^{-1}} \notag\\
    \le & \sqrt{\sum_{\tau\in[t-1]}\|(\tilde{\Sigma}^{\operatorname{res}}_{t-1})^{-\frac{1}{2}}(\hat{U}^{\operatorname{res}})^{\top}\varphi(X_{\tau}, a_{\tau})\|^2_2}\sqrt{\sum_{\tau\in[t-1]}(\varphi(X_{\tau}, a_{\tau})^{\top}(\Pi^{\mathcal{S}^{\operatorname{res}}}-\hat{\Pi}^{\mathcal{S}^{\operatorname{res}}})\delta^{\operatorname{res}}_{t})^2}\notag\\
    \le &  \sqrt{p^{\operatorname{res}}}\sqrt{t\max_{\tau\in[t-1]}\|\varphi(X_{\tau}, a_{\tau})\|_2^2\|\delta^{\operatorname{res}}_t\|_2^2\|\Pi^{\mathcal{S}^{\operatorname{res}}}-\hat{\Pi}^{\mathcal{S}^{\operatorname{res}}}\|_{\operatorname{op}}^2} \notag\\
    \le & 
LM\Delta\Pi\sqrt{p^{\operatorname{res}}t}\eqqcolon \hat{\rho}^{\operatorname{res}}_2 \label{eq:delta_bound_delta_term_end}.
\end{align}
The first inequality follows from Cauchy-Schwarz inequality. The second inequality uses the fact that
\begin{align*}
    & \sum_{\tau\in[t-1]}\|(\tilde{\Sigma}^{\operatorname{res}}_{t-1})^{-\frac{1}{2}}(\hat{U}^{\operatorname{res}})^{\top}\varphi(X_{\tau}, a_{\tau})\|^2_2 \\
     = &\sum_{\tau\in[t-1]}\varphi(X_{\tau}, a_{\tau})^{\top}\hat{U}^{\operatorname{res}}(\tilde{\Sigma}^{\operatorname{res}}_{t-1})^{-1}(\hat{U}^{\operatorname{res}})^{\top}\varphi(X_{\tau}, a_{\tau}) \\
    = & \operatorname{trace}\left( \sum_{\tau\in[t-1]}(\hat{U}^{\operatorname{res}})^{\top}\varphi(X_{\tau}, a_{\tau})\varphi(X_{\tau}, a_{\tau})^{\top}\hat{U}^{\operatorname{res}}(\tilde{\Sigma}^{\operatorname{res}}_{t-1})^{-1}\right) \\
     \le & \operatorname{trace}\left( \left(\sum_{\tau\in[t-1]}(\hat{U}^{\operatorname{res}})^{\top}\varphi(X_{\tau}, a_{\tau})\varphi(X_{\tau}, a_{\tau})^{\top}\hat{U}^{\operatorname{res}}+\lambda I_{p^{\operatorname{res}}}\right)(\tilde{\Sigma}^{\operatorname{res}}_{t-1})^{-1}\right) \\
      = & \operatorname{trace}(I_{p^{\operatorname{res}}})\\
      = & p^{\operatorname{res}}.     
\end{align*}
Under Assumption~\ref{ass:subspace_dec_error}, for all $\eta\in(0,1)$, with probability at least $1-\eta$, \eqref{eq:delta_bound_delta_term_end} is $O\left(\sqrt{\frac{p^{\operatorname{res}}t\log(p/\eta)}{T_0}}\right)$. 
To bound the term in \eqref{eq:delta_bound_beta_term} we use the result in Lemma~\ref{lem:beta_inv_bound} on the estimation error for the invariant component, since
\begin{align*}
    & \|\sum_{\tau\in[t-1]}(\hat{U}^{\operatorname{res}})^{\top}\varphi(X_{\tau}, a_{\tau})\varphi(X_{\tau}, a_{\tau})^{\top}(\beta^{\operatorname{inv}}-\hat{\beta}^{\operatorname{inv}})\|_{(\tilde{\Sigma}^{\operatorname{res}}_{t-1})^{-1}} \\
      = &  \|(\tilde{\Sigma}^{\operatorname{res}}_{t-1})^{-\frac{1}{2}}\sum_{\tau\in[t-1]}(\hat{U}^{\operatorname{res}})^{\top}\varphi(X_{\tau}, a_{\tau})\varphi(X_{\tau}, a_{\tau})^{\top}(\beta^{\operatorname{inv}}-\hat{\beta}^{\operatorname{inv}})\|_{2} \\
      \le & \sqrt{\sum_{\tau\in[t-1]}\|(\tilde{\Sigma}^{\operatorname{res}}_{t-1})^{-\frac{1}{2}}(\hat{U}^{\operatorname{res}})^{\top}\varphi(X_{\tau}, a_{\tau})\|_2^2} \sqrt{\sum_{\tau\in[t-1]}(\varphi(X_{\tau}, a_{\tau})^{\top}(\beta^{\operatorname{inv}}-\hat{\beta}^{\operatorname{inv}}))^2}\\
      \le & \sqrt{p^{\operatorname{res}}}\sqrt{t\max_{\tau\in[t-1]}\|\varphi(X_{\tau}, a_{\tau})\|_2^2\|\beta^{\operatorname{inv}}-\hat{\beta}^{\operatorname{inv}}\|_2^2}\\
      \le & L\|\beta^{\operatorname{inv}}-\hat{\beta}^{\operatorname{inv}}\|_2\sqrt{p^{\operatorname{res}}t} \eqqcolon \hat{\rho}^{\operatorname{res}}_3. 
\end{align*}
Taking the union bound, we have that, with probability at least $1-2\eta$,
\begin{equation}
    \|\hat{\delta}^{\operatorname{res}}_t-\hat{\Pi}^{\mathcal{S}^{\operatorname{res}}}\delta^{\operatorname{res}}_t\|_{\hat{\Sigma}_{t-1}}\le \sum_{i=1}^3 \hat{\rho}^{\operatorname{res}}_i \eqqcolon \hat{\rho}^{\operatorname{res}}_t(\eta, L, M).
    \label{eq:delta_res_radius_full}
\end{equation}
To conclude the proof of the lemma it is sufficient to replace $\eta$ with $\eta'\coloneqq \eta/2$ in all above terms.  
\end{proof}

\begin{proof}[Proof of Theorem~\ref{thm:full_regret_bound}]
    We start by considering a decomposition of the instantaneous regret similar to the one in \eqref{eq:inst_reget_isd}. Now, the true subspace decomposition is unknown, so $\hat{\mathcal{C}}^\beta$ and $\hat{\mathcal{C}}^{\delta}_t$ are confidence sets containing, with probability at least $1-\eta$, $\hat{\Pi}^{\mathcal{S}^{\operatorname{inv}}}\beta^{\operatorname{inv}}$  and $\hat{\Pi}^{\mathcal{S}^{\operatorname{res}}}\delta^{\operatorname{res}}_t$, respectively. 
      As before, we denote by $(\bar{\beta}_t, \bar{\delta}_t) = \argmax_{\beta\in\hat{\mathcal{C}}^{\beta}, \delta\in\hat{\mathcal{C}}^{\delta}_t}\varphi(X_t, a_t)^{\top}(\beta+\delta)$, and by $\bar{\gamma}_t=\bar{\beta}_t+\bar{\delta}_t$ the parameter at which the upper confidence bound for the predicted reward at time $t$ is achieved. Moreover, recall that $\tilde{\Sigma}^{\operatorname{res}}_t = (\hat{U}^{\operatorname{res}})^{\top}\hat{\Sigma}_t \hat{U}^{\operatorname{res}}$ and $\tilde{\Sigma}^{\operatorname{inv}}_{[-T_0]} = (\hat{U}^{\operatorname{inv}})^{\top}\hat{\Sigma}_{[-T_0]} \hat{U}^{\operatorname{inv}}$. Then, at time $t\in[T]$, $\operatorname{reg}_t$ is such that
    \begin{align*}
        \operatorname{reg}_t & \coloneqq (\varphi(X_t, a^*_t)-\varphi(X_t, a_t))^{\top}\gamma_{0,t} \\
        & \le \varphi(X_t, a_t)^{\top}(\bar{\gamma}_t - \gamma_{0,t})\\
        & = \varphi(X_t, a_t)^{\top}(\bar{\beta}_t-\beta^{\operatorname{inv}}) + \varphi(X_t, a_t)^{\top}(\bar{\delta}_t-\delta^{\operatorname{res}}_t)\\
        & =  \varphi(X_t, a_t)^{\top}(\bar{\beta}_t- \hat{\Pi}^{\mathcal{S}^{\operatorname{inv}}}\beta^{\operatorname{inv}}+\hat{\Pi}^{\mathcal{S}^{\operatorname{inv}}}\beta^{\operatorname{inv}}- \beta^{\operatorname{inv}}) \\
        &+ \varphi(X_t, a_t)^{\top}(\bar{\delta}_t- \hat{\Pi}^{\mathcal{S}^{\operatorname{res}}}\delta^{\operatorname{res}}_t+\hat{\Pi}^{\mathcal{S}^{\operatorname{res}}}\delta^{\operatorname{res}}_t -\delta^{\operatorname{res}}_t)\\
        & \le \|(\hat{U}^{\operatorname{inv}})^{\top}\varphi(X_t, a_t)\|_2 \|(\hat{U}^{\operatorname{inv}})^{\top}\bar{\beta}_t-(\hat{U}^{\operatorname{inv}})^{\top}\beta^{\operatorname{inv}}\| _2 \\ 
        & + \|(\hat{U}^{\operatorname{res}})^{\top}\varphi(X_t, a_t)\|_{((\hat{U}^{\operatorname{res}})^{\top}\hat{\Sigma}_{t-1}\hat{U}^{\operatorname{res}})^{-1}} \|(\hat{U}^{\operatorname{res}})^{\top}\bar{\delta}_t-(\hat{U}^{\operatorname{res}})^{\top}\delta^{\operatorname{res}}_t\|_{(\hat{U}^{\operatorname{res}})^{\top}\hat{\Sigma}_{t-1}\hat{U}^{\operatorname{res}}}\\
        & + \|\varphi(X_t, a_t)\|_2 \|\gamma_{0,t}\|_2 \Delta\Pi \\
        & \le \|(\hat{U}^{\operatorname{inv}})^{\top}\varphi(X_t, a_t)\|_2 \tfrac{1}{\sqrt{\lambda_0 T_0}}\|(\hat{U}^{\operatorname{inv}})^{\top}\bar{\beta}_t-(\hat{U}^{\operatorname{inv}})^{\top}\beta^{\operatorname{inv}}\|_{\tilde{\Sigma}^{\operatorname{inv}}_{[-T_0]}} \\ 
        & + \|(\hat{U}^{\operatorname{res}})^{\top}\varphi(X_t, a_t)\|_{(\tilde{\Sigma}^{\operatorname{res}}_{t-1})^{-1}} \|(\hat{U}^{\operatorname{res}})^{\top}\bar{\delta}_t-(\hat{U}^{\operatorname{res}})^{\top}\delta^{\operatorname{res}}_t\|_{\tilde{\Sigma}^{\operatorname{res}}_{t-1}}\\
        & + \|\varphi(X_t, a_t)\|_2  \|\gamma_{0,t}\|_2 \Delta\Pi \\
        & \le 2 \|(\hat{U}^{\operatorname{inv}})^{\top}\varphi(X_t, a_t)\|_2 \sqrt{\tfrac{\hat{\rho}^{\operatorname{inv}}_{T_0}(\eta, L, M)}{\lambda_0 T_0}} + 2 \|(\hat{U}^{\operatorname{res}})^{\top}\varphi(X_t, a_t)\|_{(\tilde{\Sigma}^{\operatorname{res}}_{t-1})^{-1}} \sqrt{\hat{\rho}^{\operatorname{res}}_t(\eta, L, M)} \\
        & +  \|\varphi(X_t, a_t)\|_2 \|\gamma_{0,t}\|_2 \Delta\Pi.
    \end{align*}
    Lemma~\ref{lem:beta_inv_bound} implies that we can define $\hat{\mathcal{C}}^{\operatorname{\beta}}$ to contain $\hat{\Pi}^{\mathcal{S}^{\operatorname{inv}}}\beta^{\operatorname{inv}}$ with probability at least $1-\eta$ so that
    $\sqrt{\hat{\rho}^{\operatorname{inv}}_{T_0}(\eta, L, M)}$ is \begin{align*}
         O\left(\sqrt{p^{\operatorname{inv}}\log\left(\tfrac{1}{p^{\operatorname{inv}}\lambda_0}\right)+\log(\tfrac{1}{\eta})}\right)+
         O\left(\sqrt{p^{\operatorname{inv}}\log(\tfrac{p}{\eta})}\right)+
         O\left(\sqrt{\tfrac{1}{\lambda_0}\log(\tfrac{p}{\eta})}\right).
    \end{align*}
    Lemma~\ref{lem:delta_res_bound} implies instead that we can define $\hat{\mathcal{C}}^{\delta}_t$ to contain $\hat{\Pi}^{\mathcal{S}^{\operatorname{res}}}\delta^{\operatorname{res}}_t$ so that $\sqrt{\hat{\rho}^{\operatorname{res}}_t(\eta, L, M)}$ is, for all $\eta\in(0, \frac{1}{2})$ with probability at least $1-2\eta$
    \begin{align}
        O\left(\sqrt{p^{\operatorname{res}}\log(\tfrac{t}{p^{\operatorname{res}}})+\log(\tfrac{1}{\eta})}\right)
        + O\left(\sqrt{\tfrac{p^{\operatorname{res}}t}{T_0}}\left(\sqrt{\log(\tfrac{p}{\eta})} + \sqrt{\tfrac{\hat{\rho}^{\operatorname{inv}}_{T_0}(\eta, L, M)}{\lambda_0}} \right)\right).\label{eq:rho_res_bound}
    \end{align}
    Moreover, by Assumption~\ref{ass:subspace_dec_error}, we have that, for all $\eta\in(0,1)$, with probability at least $1-\eta$,
    \begin{align*}
        \|\varphi(X_t, a_t)\|_2 \Delta\Pi\|\gamma_{0,t}\|_2 \le LM \Delta\Pi
    \end{align*}
    is $O(\sqrt{\log(p/\eta)/T_0})$. 
    Taking now the sum over the time horizon $T$ we obtain that
    \begin{align*}
        \operatorname{Reg}_T & = \sum_{t=1}^T \operatorname{reg}_t \\
        & \le \sum_{t=1}^T \left(2 \|(\hat{U}^{\operatorname{inv}})^{\top}\varphi(X_t, a_t)\|_2 \sqrt{\tfrac{\hat{\rho}^{\operatorname{inv}}_{T_0}(\eta, L, M)}{\lambda_0 T_0}} +  \|\varphi(X_t, a_t)\|_2 \|\gamma_{0,t}\|_2\Delta\Pi \right)\\
        & + \sum_{t=1}^T 2 \|(\hat{U}^{\operatorname{res}})^{\top}\varphi(X_t, a_t)\|_{(\tilde{\Sigma}^{\operatorname{res}}_{t-1})^{-1}} \sqrt{\rho^{\operatorname{res}}_t}
    \end{align*}
    We consider these two sums separately. For the first, we obtain that 
    \begin{align*}
       & \sum_{t=1}^T \left(2 \|(\hat{U}^{\operatorname{inv}})^{\top}\varphi(X_t, a_t)\|_2 \sqrt{\tfrac{\hat{\rho}^{\operatorname{inv}}_{T_0}(\eta, L, M)}{\lambda_0 T_0}} +  \|\varphi(X_t, a_t)\|_2 \|\gamma_{0,t}\|_2 \Delta\Pi\right) \\
    \le & T\left(2 L \sqrt{\tfrac{\hat{\rho}^{\operatorname{inv}}_{T_0}(\eta, L, M)}{\lambda_0 T_0}}  + LM \Delta\Pi\right).
    \end{align*}
From the results above, we obtain that this term is, for all $\eta\in(0, \frac{1}{5})$, with probability at least $1-5\eta$,
\begin{align*}
     O\left(\tfrac{T}{\sqrt{T_0\lambda_0}}\sqrt{p^{\operatorname{inv}}+\log(\tfrac{1}{\eta})}\right) 
         +  O\left(\tfrac{T}{\sqrt{T_0\lambda_0}} \sqrt{p^{\operatorname{inv}}\log(\tfrac{p}{\eta})}\right) 
        + O\left(\tfrac{T}{\sqrt{T_0\lambda_0}} \sqrt{\max\{\lambda_0, \tfrac{1}{\lambda_0}\}\log(\tfrac{p}{\eta})}\right).
\end{align*}
For the remaining part of the cumulative regret, we have that
\begin{align*}
    \sum_{t=1}^T 2 \|(\hat{U}^{\operatorname{res}})^{\top}\varphi(X_t, a_t)\|_{(\tilde{\Sigma}^{\operatorname{res}}_{t-1})^{-1}} \sqrt{\hat{\rho}^{\operatorname{res}}_t(\eta, L, M)} \le 2 \sqrt{T \hat{\rho}^{\operatorname{res}}_T(\eta, L, M)\sum_{t=1}^T \|(\hat{U}^{\operatorname{res}})^{\top}\varphi(X_t, a_t)\|_{(\tilde{\Sigma}^{\operatorname{res}}_{t-1})^{-1}}^2 }
\end{align*}
As in the oracle case (see the proof of Lemma~\ref{lem:delta_res_bound}), we can use \eqref{eq:cumreg_res_bouund} and Lemma~\ref{lem:bandits_alg_contexts_bound} and obtain
\begin{align*}
     \sum_{t=1}^T 2 \|(\hat{U}^{\operatorname{res}})^{\top}\varphi(X_t, a_t)\|_{(\tilde{\Sigma}^{\operatorname{res}}_{t-1})^{-1}} \sqrt{\hat{\rho}^{\operatorname{res}}_t(\eta, L, M)} \le C^{\operatorname{res}}\sqrt{\hat{\rho}^{\operatorname{res}}_T(\eta, L, M) T p^{\operatorname{res}}\log \left(1+\tfrac{TL^2}{\lambda p^{\operatorname{res}}}\right)}.
\end{align*}
Using \eqref{eq:rho_res_bound}, we have that the above term is, for all $\eta\in(0, \frac{1}{2})$, with probability at least $1-2\eta$,
\begin{align*}
     & O\left(p^{\operatorname{res}}\sqrt{ T\log \left(1+\tfrac{TL^2}{\lambda p^{\operatorname{res}}}\right)\left(\log(\tfrac{T}{p^{\operatorname{res}}})+\log(\tfrac{1}{\eta})\right)}\right) \\
        + & O\left(\tfrac{p^{\operatorname{res}}T}{\sqrt{T_0}}\sqrt{\log \left(1+\tfrac{TL^2}{\lambda p^{\operatorname{res}}}\right)}\left(\sqrt{\log(\tfrac{p}{\eta})} + \sqrt{\tfrac{\hat{\rho}^{\operatorname{inv}}_{T_0}(\eta, L, M)}{\lambda_0}} \right)\right).
\end{align*}
 As a result, up to logarithmic terms, we obtain that the cumulative regret for the ISD-linUCB algorithm is, for all $\eta\in(0, \frac{1}{11})$, with probability at least $1-11\eta$
\begin{equation*}
    \tilde{O}\left(\sqrt{T}\left(p^{\operatorname{res}} + p^{\operatorname{res}}\sqrt{\tfrac{T}{\lambda_0 T_0}}\left(\sqrt{p^{\operatorname{inv}}}+\sqrt{\max\{\lambda_0, \tfrac{1}{\lambda_0}\}}\right)\right)\right).
\end{equation*}
\end{proof}

\section{Lower bound}
\label{sec:lower_bound}
\subsection{Lower bound for LinUCB}
\citet{lattimore2020bandit} show the lower bound $\Omega(p\sqrt{T})$ for the LinUCB algorithm. In the linear model, they consider Gaussian noise with unit variance. Then, for a context-action space $\mathcal{X}\times\mathcal{A}$ and a linear parameter $\gamma\in\mathbb{R}^p$, the regret of a policy is 
\begin{equation*}
    \operatorname{Reg}_T(\mathcal{X}\times\mathcal{A}, \gamma) = T\max_{\varphi\in\mathcal{X}\times\mathcal{A}} \varphi^{\top}\gamma - \mathbb{E}_{\gamma} \left[\sum_{t=1}^T R_t^{a_t}\right]
 \end{equation*}
where the expectation is taken with respect to the measure on the rewards induced by the interaction of the policy with the bandit parametrized by $\gamma$. 
The idea is to find a worst case instance of the bandit problem, which is characterized by the context-action space and by the parameter $\gamma$, with cumulative regret at least of the order of $p\sqrt{T}$.
\citet{lattimore2020bandit} show the result for two different choices of the context-action space. The first is the hypercube, where $\mathcal{X}\times\mathcal{A}= [-1, 1]^p$. The second choice is the unit sphere, namely $\mathcal{X}\times\mathcal{A}=\{x\in\mathbb{R}^p\mid \|x\|_2\le 1\}$. In the first case, the proof is simpler because the value of the context-action feature in one direction is not constrained by the values in the remaining directions: we focus on this proof for now.
More in detail, the proof for $\mathcal{X}\times\mathcal{A}= [-1, 1]^p$ shows that there exists a parameter $\gamma\in\Gamma\coloneqq \left\{\pm\frac{1}{\sqrt{T}}\right\}^p$ such that $\operatorname{Reg}_T(\mathcal{X}\times\mathcal{A}, \gamma)$ is $\Omega(p\sqrt{T})$.
The proof starts by considering the relative entropy (or KL-divergence) between the probability measures $\mathbb{P}_{\gamma}, \mathbb{P}_{\gamma'}$ on the rewards, induced by the interaction of a fixed policy, i.e., a fixed sequence of context-action features $\{\varphi(X_t, a_t)\}_{t=1}^T$, with two bandits parametrized by $\gamma, \gamma' \in \Gamma$ respectively (with noise $\epsilon_t\sim\mathcal{N}(0,1)$), that is,
\begin{equation*}
    D(\mathbb{P}_{\gamma}, \mathbb{P}_{\gamma'}) = \frac{1}{2} \sum_{t=1}^T \mathbb{E}_{\gamma}[(\varphi(X_t, a_t)^{\top}(\gamma-\gamma'))^2].
\end{equation*}
Denoting with the subscript $i$ the $i$-th component of a vector, they then define for $i\in[p]$ and $\gamma\in\Gamma$
\begin{equation*}
    q_{\gamma_i}\coloneqq \mathbb{P}_{\gamma}\left(\sum_{t=1}^T \mathbbm{1}\{\operatorname{sign}(\varphi_i(X_t, a_t) \neq \operatorname{sign}(\gamma_i)\}\ge \frac{T}{2} \right),
\end{equation*}
that is, the probability that the number of times that the sign of the $i$-th component of the selected context-action feature is not optimal exceeds half of the total learning horizon. 
For a fixed $i\in[p]$ and $\gamma\in\Gamma$, they further fix $\gamma'$ as the parameter that equals $\gamma$ in all its components except for the $i$-th, which is set to $\gamma'_i=-\gamma_i$. Then, it holds by a property of the relative entropy (Bretangolle-Huber inequality) that
\begin{equation*}
    q_{\gamma_i} + q_{\gamma'_i} \ge \frac{1}{2}\operatorname{exp}\left(-D(\mathbb{P}_{\gamma}, \mathbb{P}_{\gamma'})\right) = \frac{1}{2}\operatorname{exp}\left(-\frac{1}{2} \sum_{t=1}^T \mathbb{E}_{\gamma}[(\varphi(X_t, a_t)^{\top}(\gamma-\gamma'))^2]\right) \ge \frac{1}{2} \operatorname{exp}(-2),
\end{equation*}
where the last inequality follows from the definition of the spaces $\mathcal{A}\times\mathcal{X}$ and $\Gamma$. 
Averaging over all possible parameters $\gamma\in\Gamma$, this implies that
\begin{equation*}
    \frac{1}{|\Gamma|}\sum_{\gamma\in\Gamma} \sum_{i\in[p]} q_{\gamma_i} =  \frac{1}{|\Gamma|} \sum_{i\in[p]} \sum_{\gamma\in\Gamma} q_{\gamma_i} \ge \frac{1}{|\Gamma|}\sum_{i\in[p]} \frac{|\Gamma|}{2} \frac{1}{2} \exp{(-2)} = \frac{p}{4}\operatorname{exp}(-2).
\end{equation*}
Finally, this implies that there exists at least one parameter $\gamma\in\Gamma$ such that $\sum_{i\in[p]}  q_{\gamma_i} \ge \frac{p}{4}\operatorname{exp}(-2)$. For this choice of $\gamma$, the regret is 
\begin{align*}
    \operatorname{Reg}_T(\mathcal{X}\times\mathcal{A}, \gamma) & = \mathbb{E}_{\gamma}\left[ \sum_{t=1}^T\sum_{i=1}^p(\operatorname{sign}(\gamma_i)-\varphi_i(X_t, a_t))\gamma_i\right] \\
    & \ge \frac{1}{\sqrt{T}} \sum_{i=1}^p\mathbb{E}_{\gamma} \left[ \sum_{t=1}^T\mathbbm{1}\{\operatorname{sign}(\varphi_i(X_t, a_t))\neq\operatorname{sign}(\gamma_i)\}\right] \\
    & \ge \frac{\sqrt{T}}{2} \sum_{i=1}^p \mathbb{P}_{\gamma}\left(\sum_{t=1}^T\mathbbm{1}\{\operatorname{sign}(\varphi_i(X_t, a_t))\neq\operatorname{sign}(\gamma_i)\}\ge \frac{T}{2} \right) \\
    & = \frac{\sqrt{T}}{2} \sum_{i=1}^p q_{\gamma_i} \\
    & \ge \frac{\exp{-2}}{8} p\sqrt{T}.
\end{align*}
The first equality follows from the fact that $\max_{\varphi\in [-1, 1]^{p}} \varphi^{\top}\theta = \operatorname{sign}(\theta)^{\top}\theta$. The first inequality follows from the definition of the space $\Gamma$ and a case-based analysis on the value of the summation terms. The second inequality uses Markov's inequality. The last equality follows from the choice of $\gamma$. 
This concludes the proof for the lower bound.

\subsection{Lower bound for ISD with oracle subspaces and oracle invariant component}
Consider now the case in which we have oracle access to the invariant subspace decomposition $(\mathcal{S}^{\operatorname{inv}}, \mathcal{S}^{\operatorname{res}})$ and the invariant parameter $\beta^{\operatorname{inv}}$.  We can modify the proof for the standard linear case above by  considering a parameter space where the parameters only differ in their residual components, that is, their projection onto $\mathcal{S}^{\operatorname{res}}$, while the invariant component remains fixed. Under this additional constraint, the effective dimension of $\Gamma$ is reduced from $p$ to $p^{\operatorname{res}}$,  leading to the lower bound $\Omega(p^{\operatorname{res}}\sqrt{T})$. 

More in detail, we consider again the context-action space $\mathcal{X}\times\mathcal{A}=[-1, 1]^p$. We now define the parameter space as $\Gamma \coloneqq \left\{\gamma\in\mathbb{R}^p\mid \gamma = \beta^{\operatorname{inv}} + \delta^{\operatorname{res}}, (U^{\operatorname{res}})^{\top}\delta^{\operatorname{res}}\in \left\{\pm \frac{1}{\sqrt{T}}\right\}^{p^{\operatorname{res}}}\right\}$. 
Then, we have that for two parameters $\gamma, \gamma'\in\Gamma$, 
\begin{equation*}
     D(\mathbb{P}_{\gamma}, \mathbb{P}_{\gamma'}) = \frac{1}{2} \sum_{t=1}^T \mathbb{E}_{\gamma}[(\varphi(X_t, a_t)^{\top}(\delta^{\operatorname{res}}-\delta'^{\operatorname{res}}))^2].
\end{equation*}
For all $t\in[T]$ let $\tilde{\varphi}^{\operatorname{res}}(X_t, a_t)\coloneqq (U^{\operatorname{res}})^{\top}\varphi(X_t, a_t)$ and $\tilde{\delta}^{\operatorname{res}}\coloneqq (U^{\operatorname{res}})^{\top}\delta^{\operatorname{res}}$. As above, we can define for $i\in[p^{\operatorname{res}}]$ and $\gamma\in\Gamma$
\begin{equation*}
    q_{\gamma_i} \coloneqq \mathbb{P}_{\gamma} \left(\sum_{t=1}^T \mathbbm{1}\{\operatorname{sign}(\tilde{\varphi}^{\operatorname{res}}_i(X_t, a_t)) \neq \operatorname{sign}(\tilde{\delta}^{\operatorname{res}}_i)\} \ge \frac{T}{2}\right).
\end{equation*}
For a fixed $i\in[p^{\operatorname{res}}]$ and $\gamma\in\Gamma$, define $\gamma'\in\Gamma$ as the parameter obtained from $\gamma$ by setting $\tilde{\delta}'^{\operatorname{res}}_i = - \tilde{\delta}^{\operatorname{res}}_i$ (and otherwise equal to $\gamma$). Then,
\begin{align*}
    q_{\gamma_i} + q_{\gamma'_i} \ge \frac{1}{2}\exp{\left(-\frac{1}{2}\sum_{t=1}^T\mathbb{E}_{\gamma}[(\varphi(X_t, a_t)^{\top}(\delta^{\operatorname{res}}-\delta'^{\operatorname{res}}))^2]\right)}\ge \frac{1}{2}\exp{(-2)}
\end{align*}
which again follows from the definition of the spaces $\mathcal{X}\times\mathcal{A}$ and $\Gamma$. 
Taking the average over all possible parameters $\gamma\in\Gamma$, we obtain
\begin{equation*}
    \frac{1}{|\Gamma|}\sum_{\gamma\in\Gamma}\sum_{i\in[p^{\operatorname{res}}]}q_{\gamma_i} \ge \frac{1}{|\Gamma|}\sum_{i\in[p^{\operatorname{res}}]} \frac{|\Gamma|}{2} \frac{1}{2}\exp{(-2)} = p^{\operatorname{res}}\frac{\exp{(-2)}}{4},
\end{equation*}
meaning that there exists $\gamma\in\Gamma$ such that $\sum_{i\in[p^{\operatorname{res}}]} q_{\gamma_i} \ge p^{\operatorname{res}}\frac{\exp{(-2)}}{4}$. The last part of the proof follows exactly the same steps as the one in the standard case, leading this time to a regret that is $\Omega(p^{\operatorname{res}}\sqrt{T})$.

\end{document}